\documentclass[letterpaper,journal]{IEEEtran}
\usepackage[utf8]{inputenc}
\usepackage[T1]{fontenc}
\usepackage{mathtools}
\usepackage{amssymb,amsfonts,amsthm}
\usepackage{bm}
\usepackage{diagbox}
\usepackage{graphicx}
\usepackage{caption}
\usepackage{subcaption}
\usepackage{booktabs}
\usepackage{multirow}
\usepackage{amsmath}
\usepackage{threeparttable}
\usepackage{tabularray}
\UseTblrLibrary{booktabs}
\usepackage{adjustbox}
\usepackage{array}
\usepackage{siunitx}
\usepackage{relsize}
\usepackage[normalem]{ulem}
\usepackage{xpatch}
\usepackage{stfloats}
\usepackage{textcomp}
\usepackage{microtype}
\usepackage{nicefrac}
\usepackage{enumitem}
\usepackage{comment}
\usepackage{tikz}
\usepackage{pifont}
\usepackage{url}
\usepackage[table]{xcolor}
\usepackage{hyperref}
\usepackage{balance}
\usepackage{verbatim}

\hypersetup{hidelinks,
	pdftitle={HiER-BLS: A Hierarchy-Guided and Error-Correcting Robust Incremental Broad Learning System},
	pdfauthor={Gongli Zhang; C. L. Philip Chen; Zhulin Liu},
	colorlinks=true,
	allcolors=black,
	pdfstartview=Fit}

\definecolor{lime}{HTML}{A6CE39}
\DeclareRobustCommand{\orcidicon}{
\begin{tikzpicture}
\draw[lime, fill=lime] (0,0)
circle[radius=0.16]
node[white]{{\fontfamily{qag}\selectfont \tiny \.{I}D}};
\end{tikzpicture}
\hspace{-2mm}
}
\foreach \x in {A, ..., Z}{
\expandafter\xdef\csname orcid\x\endcsname{\noexpand\href{https://orcid.org/\csname orcidauthor\x\endcsname}{\noexpand\orcidicon}}
}

\newtheorem{theorem}{Theorem}

\newtheorem{corollary}{Corollary}
\newtheorem{proposition}{Proposition}

\newcommand{\arraystretchparam}{1.2}

\begin{document}
\title{HiER-BLS: A Hierarchy-Guided and Error-Correcting Robust Incremental Broad Learning System}
\author{\IEEEauthorblockN{Gongli Zhang\hspace{-1.5mm}\orcidA{}, C. L. Philip Chen\hspace{-1.5mm}\orcidB{}, and Zhulin Liu\hspace{-1.5mm}\orcidC{}}% <-this
% stops a space
    \thanks{This work was funded in part by the National Natural Science Foundation of China grant under number 62536004, 62222603, in part by the Key-Area Research and Development Program of Guangdong Province under number 2023B0303030001, in part by the Program for Guangdong Introducing Innovative and Entrepreneurial Teams (2019ZT08X214), and in part by the Science and Technology Program of Guangzhou under number 2024A04J6310, and in part by the Fundamental Research Funds for the Central Universities 2025ZYGXZR021.
    \emph{(Corresponding author: Zhulin Liu.)}}% <-this % stops a space

    \thanks{The authors are with the Guangdong Provincial Key Laboratory of Computational AI Models and Cognitive Intelligence, the School of Computer Science and Engineering, South China University of Technology, Guangzhou 510006, China, also with the Pazhou Lab, Guangzhou 510335, China, and also with the Engineering Research Center of the Ministry of Education on Health Intelligent Perception and Paralleled Digital-Human, Guangzhou 510641, China (e-mail: csglzhang@mail.scut.edu.cn; philip.chen@ieee.org; liuzhl@scut.edu.cn).}
}

\markboth{HiER-BLS}{Zhang et al.: HiER-BLS: A Hierarchy-Guided and Error-Correcting Robust Incremental Broad Learning System}

\maketitle

\begin{abstract}
Broad Learning System (BLS) supports analytical training and incremental expansion, but its growth needs guidance on which inputs new blocks should learn from. Weight errors pose a further challenge by displacing learned outputs across class boundaries. We propose HiER-BLS to couple hierarchy-guided representation growth with error-correcting learning. Successive blocks focus on inputs selected by feature importance and correlation while preserving earlier representations. The evolving branch guides encoded learners through subspace size and sample confidence, so its learning experience informs both their feature views and supervision. For finite broad readouts, we show how codeword correlations transform fitted class scores. Prediction preservation depends on the distance from the actual output to the nearest decoding boundary relative to the model's sensitivity to weight errors. Experiments on five image and five tabular datasets demonstrate improved classification performance over representative BLS variants. Component studies show that hierarchy guidance benefits the encoded branch even when the guiding branch has lower standalone accuracy, with further gains from combining their scores. Longer codes continue to improve accuracy under stronger Gaussian weight errors after clean accuracy has largely saturated.
\end{abstract}


\section*{NOMENCLATURE}
\begin{description}[leftmargin=1.8cm,labelsep=0cm,style=nextline,font=\normalfont]
\item[$N,M,C$] Numbers of training samples, input features, and classes.
\item[$n,m$] Numbers of feature and enhancement windows.
\item[$\mathbf X,\mathbf Y$] Input matrix and one-hot targets.
\item[$\mathbf1_N$] All-ones column vector of length $N$.
\item[$\mathbf T$] Fixed target matrix; may include confidence scaling.
\item[$\mathbf D$] Nonnegative diagonal sample-weight matrix.
\item[$u_{t,ic},w_{t,i}$] Incorrect-class pair weights and normalized sample weights.
\item[$\mathcal L_{\mathrm{CE}}$] Cross-entropy of class scores and ground-truth labels.
\item[$\lambda$] Regularization coefficient for readout fitting and subspace alignment.
\item[$d$] Total scalar width of a fixed broad representation.
\item[$\mathbf a,\mathbf W$] Row of broad features and fitted output weights.
\item[$\mathbf K$] Fixed linear map from targets to output weights.
\item[$L,\Omega$] Code length and $C\times L$ codebook.
\item[$\mathbf G_\Omega$] Code Gram matrix $\Omega\Omega^\top$.
\item[$\boldsymbol\omega_c$] Codeword of class $c$.
\item[$\mathbf g,\widetilde{\mathbf g}$] Clean and perturbed encoded output rows.
\item[$\|\cdot\|_{\mathrm{op}}$] Matrix spectral norm.
\item[$\varepsilon_a,\varepsilon_W$] Node-vector and readout-weight perturbation budgets.
\item[$r(\mathbf x)$] Output-perturbation bound normalized by $\sqrt L$.
\item[$\rho(\mathbf g;\Omega)$] Normalized distance to the nearest decoding boundary.
\item[$D_*,U$] Nearest codeword separation and output-to-codeword distance, normalized by $\sqrt L$.
\item[$v_i,H$] Feature importance and its entropy-based score.
\item[$\mathbf{FCM}$] Matrix of feature-correlation magnitudes.
\item[$\mathrm{LP}_i,\mathrm{GP}_i$] Local ordering and global sampling priorities.
\item[$d_t,\mathbf X_t$] Selected input width and subspace at step $t$.
\item[$\mathbf Z_t,\mathbf H_t$] Feature and enhancement blocks.
\item[$r_t,s_t^{\mathrm h},q_t$] Scalar feature, enhancement, and total widths of block $t$.
\item[$R_t$] Accumulated scalar width $\sum_{s=1}^tq_s$.
\item[$\mathbf F_t,\mathbf A_t$] Unweighted and sample-weighted broad blocks.
\item[$\overline{\mathbf A}_t$] Concatenation of all weighted blocks through step $t$.
\item[$\mathbf L_t$] Cholesky factor of the accumulated regularized Gram matrix.
\item[$P$] Number of codebooks in the initial pool.
\item[$\mathbf S,L_{\mathrm{tot}}$] Concatenated base outputs and retained total code length.
\item[$L_{b,0},L_0$] Initial length of base $b$ and total columns introduced before pruning.
\item[$\mathrm{RG}_j$] Validation-accuracy gain for deleting column $j$.
\item[$\omega$] Validation-dependent branch fusion weight.
\end{description}

\section{Introduction}
\label{sec:introduction}
\IEEEPARstart{B}{road} Learning System (BLS)~\cite{chen2017broad,2020FuzzyBLS} offers an analytical learning framework with a flat network of feature and enhancement nodes. Its incremental algorithms expand the representation while reusing the output-layer computation of the existing model~\cite{chen2018universal,gong2021research}. As the model expands, the construction of new feature blocks and the reliability of the resulting predictions become central design concerns.

Recent work has improved several aspects of incremental broad learning. Zhong et al.~\cite{zhong2024robust} use a memory-matrix mechanism to handle data streams of uncertain scale. Chen et al.~\cite{chen2025efficient} develop inverse-matrix-free updates. Fu et al.~\cite{fu2022task} and Du et al.~\cite{du2023class} extend BLS to task- and class-incremental learning, respectively. These advances address how a model can be updated as its structure or learning task changes. The information supplied to each new block also determines what the expanded representation can learn. Input coordinates differ in predictive importance, and strongly correlated coordinates may contribute overlapping information. Their importance and dependence can guide which coordinates remain available to later blocks. The overall importance profile helps determine how rapidly to narrow the input subspace. This motivates hierarchy-guided evolution. New blocks progressively focus on selected inputs, while earlier blocks remain in the accumulated broad representation.

Weight errors pose a further challenge after training. Existing robustness-oriented BLS methods address noisy features through self-encoding, graph regularization, and kernel ensembles~\cite{yang2021extracting,liu2024dynamic,yu2021progressive}, or uncertain labels through reweighting and label enhancement~\cite{liu2023broad,qian2024partial,deng2025robust}. Errors in stored weights directly perturb the learned computation. Such errors arise, for example, from limited coefficient precision and programming imperfections in nonvolatile compute-in-memory devices~\cite{yu2024error}. The resulting output displacement changes a prediction when it crosses a class decision boundary. An error-tolerant output representation must therefore be studied together with the sensitivity of the trained classifier.

Error-Correcting Output Codes (ECOCs)~\cite{dietterich1994solving,gupta2022scalable} introduce redundancy by assigning each class a codeword. Yu et al.~\cite{yu2024error} analyze ECOCs through the neural tangent kernel (NTK). Their analysis relates codeword correlations to clean decoding and normalized code distances to weight-error tolerance. The analytical readout of BLS allows this relationship to be studied directly at finite width. We show exactly how codeword correlations transform the fitted class scores, and compare the output displacement caused by weight errors with the distance to the nearest decoding boundary. This comparison links code geometry with trained-model sensitivity. Confidence-scaled targets can produce outputs whose magnitudes differ from those of standard class codewords. Measuring the boundary distance at the actual output captures this difference.

We propose HiER-BLS to couple guided representation growth with encoded ensemble learning. The Hierarchy-Guided Evolution Engine (HGEE) arranges features by importance and correlation and adapts the subspace schedule to the importance profile. Its selected subspace size and class probabilities guide the Error-Correcting Ensemble Optimizer (ECEO). ECEO uses that size to construct different feature views and gives larger encoded targets to samples that HGEE finds difficult. Thus, the evolving branch's learning experience guides both what the encoded learners see and which samples they emphasize. A code column's usefulness depends on its learned output, so Post-hoc Codebook Pruning (PCP) selects columns through joint validation performance. Broad Adaptive Decision Fusion (BADF) then learns a classifier from both branches' scores. The branches cooperate during training and contribute jointly to the final decision.

The contributions are as follows:
\begin{enumerate}
\item We propose HiER-BLS, a unified framework that couples hierarchy-guided evolution with error-correcting learning to address representation growth and prediction reliability in incremental BLS.
\item We introduce a hierarchy-guided evolution principle that allocates new feature blocks to progressively focused input subspaces while preserving the accumulated broad representation.
\item We characterize how codeword correlations transform finite broad readouts and derive a prediction-preservation condition linking actual decoding margins to trained-model sensitivity under weight errors.
\item Experiments on five image and five tabular datasets demonstrate accuracy gains from the proposed framework, with component studies revealing the value of guidance between branches and coding redundancy under weight errors.
\end{enumerate}

The remainder of this article is organized as follows. Section~\ref{sec:related_work} reviews related work. Section~\ref{sec:theory} develops the theoretical analysis. Section~\ref{sec:method} details the proposed methodology and its computational complexity. Section~\ref{sec:experiments} reports the experimental results. Finally, Section~\ref{sec:conclusion} concludes with future directions.

\section{Related Work}
\label{sec:related_work}
\subsection{Broad Learning System}
BLS constructs a finite broad representation and fits its output layer analytically~\cite{chen2017broad}.
Let the training feature matrix be $  \mathbf{X} \in \mathbb{R}^{N \times M}  $ and the corresponding target matrix be $  \mathbf{Y} \in \{0,1\}^{N \times C}  $.
BLS first constructs $n$ feature windows as follows:
\begin{equation}
    \mathbf{Z}_i = \phi(\mathbf{X}\mathbf{W}_{e_i} + \mathbf1_N\bm{\beta}_{e_i}), \quad i = 1, \ 2, \ \dots,\ n,
\end{equation}
where $\phi$ is the feature activation function, $\mathbf W_{e_i}$ is a randomly generated weight matrix, $\bm\beta_{e_i}$ is a random bias row, and $\mathbf1_N\in\mathbb R^{N\times1}$ is the all-ones column vector.
All feature windows are concatenated as $  \mathbf{Z}^n \triangleq [\mathbf{Z}_1  ,  \mathbf{Z}_2   ,   \ldots   ,   \mathbf{Z}_{n}]  $. Subsequently, $  m  $  enhancement windows are generated by
\begin{equation}
    \mathbf{H}_j = \xi\left(\mathbf{Z}^n\mathbf{W}_{h_j} + \mathbf1_N\bm{\beta}_{h_j}\right),\quad j = 1,\ 2,\ \ldots,\ m,
\end{equation}
where $\xi$ is the enhancement activation function, and $\mathbf W_{h_j}$ and $\bm\beta_{h_j}$ are also randomly generated.
Similarly, all enhancement windows are concatenated as $  \mathbf{H}^m \triangleq [\mathbf{H}_1  ,  \mathbf{H}_2 ,  \ldots  ,  \mathbf{H}_{m}]  $.
The broad feature matrix is formed as
$\mathbf A=[\mathbf Z^n,\mathbf H^m]\in\mathbb R^{N\times d}$,
where $d$ denotes its total number of columns. For a fixed broad feature
matrix, the output weights are obtained by solving the ridge-regression
problem
\begin{equation}
\mathbf W = \operatorname*{arg\,min}_{\mathbf V\in\mathbb R^{d\times C}}
\left\{
\|\mathbf A\mathbf V-\mathbf Y\|_F^2
+\lambda\|\mathbf V\|_F^2
\right\},
\label{eq:background-ridge-objective}
\end{equation}
where $\lambda>0$ is the regularization coefficient. Its unique solution is
\begin{equation}
\begin{aligned}
\mathbf W&=\mathcal K_\lambda(\mathbf A)\mathbf Y,\\
\mathcal K_\lambda(\mathbf A)
&:=(\mathbf A^\top\mathbf A+\lambda\mathbf I_d)^{-1}
\mathbf A^\top.
\end{aligned}
\label{eq:background-ridge-operator}
\end{equation}
The ridge-regression operator $\mathcal K_\lambda(\mathbf A)$ approaches
the Moore--Penrose pseudoinverse~\cite{pseudoinversebib,greville1960some} in the limit
$\mathbf A^+=\lim_{\varepsilon\downarrow0}\mathcal K_\varepsilon(\mathbf A)$.

When new feature and enhancement nodes are appended, let
$\mathbf G\in\mathbb R^{N\times q}$ collect their outputs and define
$\mathbf A_2=[\mathbf A,\mathbf G]$. With the existing columns
$\mathbf A$, the targets $\mathbf Y$, and $\lambda$ unchanged,
an exact ridge-consistent block update is given by
\cite{zhu2019ridge}
\begin{equation}
\mathcal K_\lambda(\mathbf A_2)=
\begin{bmatrix}
\mathcal K_\lambda(\mathbf A)-\mathbf E\mathbf B^\top\\
\mathbf B^\top
\end{bmatrix},
\label{eq:background-ridge-block-update}
\end{equation}
where
\begin{equation}
\mathbf E=\mathcal K_\lambda(\mathbf A)\mathbf G,
\label{eq:background-ridge-E}
\end{equation}
\begin{equation}
\mathbf R=\mathbf G-\mathbf A\mathbf E,
\label{eq:background-ridge-R}
\end{equation}
\begin{equation}
\mathbf B^\top=
\left[\mathbf R^\top\mathbf R
+\lambda\left(\mathbf I_q+\mathbf E^\top\mathbf E\right)\right]^{-1}
\mathbf R^\top.
\label{eq:background-ridge-B}
\end{equation}
The matrix inverted in \eqref{eq:background-ridge-B} is positive definite
for $\lambda>0$, irrespective of the rank of $\mathbf R$.
Consequently, the output weights for the expanded representation are
\begin{equation}
\mathbf W_2=\mathcal K_\lambda(\mathbf A_2)\mathbf Y=
\begin{bmatrix}
\mathbf W-\mathbf E\mathbf B^\top\mathbf Y\\
\mathbf B^\top\mathbf Y
\end{bmatrix}.
\label{eq:background-ridge-weight-update}
\end{equation}
These identities give the solution of the same ridge-regression problem
with $\mathbf A$ replaced by $\mathbf A_2$. In numerical computation,
the associated positive-definite linear systems can be solved using
Cholesky factors without explicitly forming their inverses.
Section~\ref{sec:hgee} extends the formulation to sample reweighting
and the accumulated representation in HGEE.

\subsection{Error-Correcting Output Codes}
ECOCs represent classes using codewords~\cite{zhang2015heuristic,dietterich2018error}, introducing redundancy to support error-tolerant classification~\cite{yu2024error}. In an ECOC codebook, each row represents a class codeword of length $L$ and each column represents a binary subproblem.
Let the ECOC codebook be $\Omega \in \{-1,1\}^{C \times L }$.
The target space is encoded as follows:
\begin{equation}
    {\ddot{\mathbf{Y}}}  = \mathbf{Y} \Omega.
\end{equation}
Accordingly, the training objective becomes
\begin{equation}
    \underset{ \theta }{ \min }\ \mathcal{L}(\ddot{f}(\mathbf{X};\theta),\ {\ddot{\mathbf{Y}}}),
\end{equation}
where $\ddot{f}(\ \cdot\ ;\theta) : \mathbb{R}^M \rightarrow \mathbb{R}^L$ is the model parameterized with $\theta$ in the ECOC space. During inference, the predicted class is obtained by
\begin{equation}
    \hat y= \underset{c}{ \arg \max }\ \ddot{f}(\mathbf{x};\theta) \cdot \Omega[c],
\end{equation}
where $\Omega[c]$ is the $c$-th class codeword. For bipolar rows, all codewords have norm $\sqrt L$, so this decision equals nearest-Euclidean decoding.

\section{Analysis of ECOC Broad Readouts}
\label{sec:theory}

ECOC changes both the output coordinates learned by a broad classifier and the geometry used to decode them. The analytical BLS solution exposes the relationship between the fitted class scores, the codebook, and the resulting decisions. Robustness then depends on how far weight errors move these outputs relative to their decoding boundaries. The model follows Chen and Liu~\cite{chen2017broad}. We extend the target-linearity analysis of Yu et al.~\cite{yu2024error} from neural tangent kernel predictors to finite ridge fits.

\subsection{ECOC-Induced Decoding Geometry}
\label{sec:ridge-coding-identity}

Consider $C\geq2$ classes and code length $L\geq1$. Let $\mathbf A\in\mathbb R^{N\times d}$ be the training representation. At an input, write its feature row as $\mathbf a(\mathbf x)\in\mathbb R^{1\times d}$. Here $d$ is the total number of feature and enhancement nodes. Fix a diagonal sample-weight matrix $\mathbf D\succeq0$, a regularization coefficient $\lambda>0$, and a target matrix $\mathbf T\in\mathbb R^{N\times C}$. The targets may be one-hot or HGEE confidence-scaled. Define
\begin{equation}\label{eq:theory-ridge-map}
\mathbf W(\mathbf T)
=(\mathbf A^\top\mathbf D\mathbf A+\lambda\mathbf I_d)^{-1}
\mathbf A^\top\mathbf D\mathbf T.
\end{equation}
It uniquely minimizes
$\|\mathbf D^{1/2}(\mathbf A\mathbf V-\mathbf T)\|_F^2+\lambda\|\mathbf V\|_F^2$.

\begin{proposition}[ECOC Decoding Geometry]\label{prop:ecoc-induced-mahalanobis-objective}
Let $\Omega\in\mathbb R^{C\times L}$ have rows $\boldsymbol\omega_c$. Under the same $\mathbf A,\mathbf D,\lambda$, the ridge solutions satisfy
\begin{equation}\label{eq:target-linearity}
\mathbf W(\mathbf T\Omega)=\mathbf W(\mathbf T)\Omega.
\end{equation}
Writing $\mathbf f_T=\mathbf a\mathbf W(\mathbf T)$ and $\mathbf g=\mathbf a\mathbf W(\mathbf T\Omega)$, we therefore have $\mathbf g=\mathbf f_T\Omega$ and
\begin{equation}\label{eq:code-gram-geometry}
\begin{split}
\|\mathbf g-\boldsymbol\omega_c\|_2^2
&=(\mathbf f_T-\mathbf e_c)\mathbf G_\Omega(\mathbf f_T-\mathbf e_c)^\top,\\
\mathbf G_\Omega&=\Omega\Omega^\top\succeq0,
\end{split}
\end{equation}
where $\mathbf e_c$ is the $c$th one-hot row vector. The quadratic form defines a norm squared if and only if $\operatorname{rank}(\Omega)=C$; otherwise it defines a seminorm squared.
\end{proposition}

\begin{proof}
The normal equation for target $\mathbf T$ is
\begin{equation}\label{eq:ridge-normal-equation}
(\mathbf A^\top\mathbf D\mathbf A+\lambda\mathbf I_d)\mathbf W(\mathbf T)
=\mathbf A^\top\mathbf D\mathbf T.
\end{equation}
Its coefficient matrix is positive definite: for every nonzero $\mathbf v\in\mathbb R^d$,
\begin{equation*}
\begin{split}
&\mathbf v^\top(\mathbf A^\top\mathbf D\mathbf A+\lambda\mathbf I_d)\mathbf v\\
&\qquad=\|\mathbf D^{1/2}\mathbf A\mathbf v\|_2^2
+\lambda\|\mathbf v\|_2^2>0.
\end{split}
\end{equation*}
Multiplying \eqref{eq:ridge-normal-equation} on the right by $\Omega$ gives the normal equation for target $\mathbf T\Omega$. Uniqueness yields \eqref{eq:target-linearity}. Substituting $\boldsymbol\omega_c=\mathbf e_c\Omega$ into the Euclidean distance gives \eqref{eq:code-gram-geometry}. Finally, for a row $\mathbf v\in\mathbb R^{1\times C}$,
$\mathbf v\mathbf G_\Omega\mathbf v^\top=\|\mathbf v\Omega\|_2^2$ is strictly positive for every nonzero $\mathbf v$ exactly when $\Omega$ has full row rank.
\end{proof}

The same Gram matrix weights the complete encoded ridge objective. For $\mathbf V\in\mathbb R^{d\times C}$, put $\mathbf E_{\mathbf V}=\mathbf A\mathbf V-\mathbf T$. Parameterizing the encoded readout by $\mathbf V\Omega$ gives the codebook-weighted quadratic form
\begin{equation}\label{eq:codebook-weighted-objective}
\begin{split}
\mathcal J_\Omega(\mathbf V)
&=\|\mathbf D^{1/2}\mathbf E_{\mathbf V}\Omega\|_F^2
+\lambda\|\mathbf V\Omega\|_F^2\\
&=\operatorname{tr}(\mathbf E_{\mathbf V}^\top
\mathbf D\mathbf E_{\mathbf V}\mathbf G_\Omega)\\
&\quad+\lambda\operatorname{tr}(\mathbf V^\top\mathbf V\mathbf G_\Omega).
\end{split}
\end{equation}
This follows from the Frobenius trace identity and cyclic invariance of the trace. Both the residual and the regularizer are weighted by $\mathbf G_\Omega$. By \eqref{eq:target-linearity}, the parameterization attains the unique encoded minimizer $\mathbf W(\mathbf T\Omega)=\mathbf W(\mathbf T)\Omega$. When $\mathbf G_\Omega$ is positive definite, the corresponding $\mathbf V$ is uniquely $\mathbf W(\mathbf T)$. When $\mathbf G_\Omega$ is singular, multiple values of $\mathbf V$ yield the same encoded weights.

The coding and distance identities also hold for fitted weights of the form $\mathbf W_{\mathbf K}(\mathbf T)=\mathbf K\mathbf T$ with a fixed $\mathbf K\in\mathbb R^{d\times N}$. This follows from $\mathbf K(\mathbf T\Omega)=(\mathbf K\mathbf T)\Omega$. This includes the block-weighted solver developed in Section~\ref{sec:hgee}.

For a bipolar codebook, every row has squared norm $L$, so the decoder is $\arg\max_c(\mathbf f_T\mathbf G_\Omega)_c$. Its score for class $c$ is
\begin{equation}\label{eq:code-score-mixing}
(\mathbf f_T\mathbf G_\Omega)_c
=L f_{T,c}+\sum_{j\ne c}f_{T,j}(\mathbf G_\Omega)_{jc}.
\end{equation}
The diagonal term preserves the class's own fitted score, while codeword correlations mix in scores from the other classes. If $\mathbf G_\Omega=\alpha\mathbf I_C$ with $\alpha>0$, decoding agrees with one-hot nearest-neighbor decoding. The off-diagonal terms explain how a nonorthogonal codebook can change the class ordering even when the broad representation is unchanged. Thus, for a fixed representation and target matrix, the effect of encoding can be understood through the fitted class scores and their transformation by the codebook. Changing the representation changes the scores on which this transformation acts.

\subsection{Output Perturbation Bounds}
\label{sec:output-perturbation-bounds}

Consider a trained readout
$\mathbf g=\mathbf a\mathbf W\in\mathbb R^{1\times L}$,
with $\mathbf W\in\mathbb R^{d\times L}$. Its perturbed output is
\begin{equation}
\widetilde{\mathbf g}=(\mathbf a+\Delta\mathbf a)
(\mathbf W+\Delta\mathbf W).
\end{equation}
The node perturbation $\Delta\mathbf a$ accounts for errors upstream of the readout. The operator norm $\|\cdot\|_{\mathrm{op}}$ below is the matrix spectral norm.

\begin{theorem}[Finite Output Perturbation Bound]\label{thm:bounded-output-bias}
If $\|\Delta\mathbf a\|_2\le\varepsilon_a$ and
$\|\Delta\mathbf W\|_{\mathrm{op}}\le\varepsilon_W$, then
\begin{equation}\label{eq:bound_of_BLS}
\begin{split}
\frac{\|\widetilde{\mathbf g}-\mathbf g\|_2}{\sqrt L}
&\le r(\mathbf x),\\
r(\mathbf x)&=
\frac{\varepsilon_a\|\mathbf W\|_{\mathrm{op}}
+(\|\mathbf a\|_2+\varepsilon_a)\varepsilon_W}{\sqrt L}.
\end{split}
\end{equation}
For a fixed trained model and input, $r(\mathbf x)\to0$ as both perturbation budgets tend to zero.
\end{theorem}

\begin{proof}
Expand the output difference as
\begin{equation}
\begin{split}
\widetilde{\mathbf g}-\mathbf g
&=\Delta\mathbf a\mathbf W+\mathbf a\Delta\mathbf W
+\Delta\mathbf a\Delta\mathbf W\\
&=\Delta\mathbf a\mathbf W
+(\mathbf a+\Delta\mathbf a)\Delta\mathbf W.
\end{split}
\end{equation}
Apply the triangle inequality and
$\|\mathbf v\mathbf M\|_2\le\|\mathbf v\|_2\|\mathbf M\|_{\mathrm{op}}$,
then divide by $\sqrt L$.
\end{proof}

\subsubsection{Trained readout norm}
\label{sec:trained-weight-norm}
For the common-weight ridge fit in \eqref{eq:theory-ridge-map}, put $\mathbf B=\mathbf D^{1/2}\mathbf A$ and $\mathbf Q=\mathbf D^{1/2}\mathbf T\Omega$. The encoded weights are
\begin{equation}
\begin{split}
\mathbf K_{\mathrm w}
&=(\mathbf B^\top\mathbf B+\lambda\mathbf I_d)^{-1}\mathbf B^\top,\\
\mathbf W=\mathbf W(\mathbf T\Omega)&=\mathbf K_{\mathrm w}\mathbf Q.
\end{split}
\end{equation}
Let $s$ range over the positive singular values of $\mathbf B$. A singular value decomposition shows that the nonzero singular values of $\mathbf K_{\mathrm w}$ are $s/(s^2+\lambda)$. Its remaining singular values are zero. Since $(s-\sqrt\lambda)^2\ge0$,
\begin{equation}
\|\mathbf K_{\mathrm w}\|_{\mathrm{op}}
\le\frac{1}{2\sqrt\lambda}.
\end{equation}
Consequently,
\begin{equation}\label{eq:ridge-weight-norm}
\|\mathbf W\|_F
=\|\mathbf K_{\mathrm w}\mathbf Q\|_F
\le\frac{\|\mathbf D^{1/2}\mathbf T\Omega\|_F}{2\sqrt\lambda}.
\end{equation}
The same argument gives
$\|\mathbf W\|_{\mathrm{op}}\le\|\mathbf Q\|_{\mathrm{op}}/(2\sqrt\lambda)$.
For a general target-linear map, the corresponding estimate is
\begin{equation}\label{eq:fixed-operator-weight-norm}
\begin{split}
\|\mathbf W_{\mathbf K}(\mathbf T\Omega)\|_F
&=\|\mathbf K\mathbf T\Omega\|_F\\
&\le\|\mathbf K\|_{\mathrm{op}}\|\mathbf T\Omega\|_F.
\end{split}
\end{equation}
The two terms in Theorem~\ref{thm:bounded-output-bias} distinguish how errors propagate through the fitted model. The readout norm $\|\mathbf W\|_{\mathrm{op}}$ amplifies upstream node errors. The node magnitude $\|\mathbf a\|_2$ controls the effect of readout-weight errors. Equations~\eqref{eq:ridge-weight-norm} and \eqref{eq:fixed-operator-weight-norm} show how confidence-scaled targets affect the readout norm. Target scaling thus changes the clean output geometry as well as the weights that determine perturbation tolerance.

\subsubsection{From parameter errors to node perturbations}
\label{sec:feature-budget}
Consider a fixed input $\mathbf x\in\mathbb R^{1\times d_{\mathrm{in}}}$ and a broad feature block
\begin{equation}\label{eq:single-input-feature-map}
\begin{split}
\mathbf z&=\phi(\mathbf x\mathbf U+\mathbf b),\\
\mathbf h&=\xi(\mathbf z\mathbf V+\mathbf c),\qquad
\mathbf a=[\mathbf z,\mathbf h].
\end{split}
\end{equation}
Here $\mathbf U\in\mathbb R^{d_{\mathrm{in}}\times n_z}$ and $\mathbf V\in\mathbb R^{n_z\times n_h}$, where $n_z$ and $n_h$ are the numbers of feature and enhancement nodes, respectively. Suppose the entrywise activations are Lipschitz with constants $\ell_\phi,\ell_\xi$, and
\begin{equation}\label{eq:upstream-parameter-budgets}
\begin{aligned}
\|\Delta\mathbf U\|_{\mathrm{op}}&\le u,&
\|\Delta\mathbf b\|_2&\le b_0,\\
\|\Delta\mathbf V\|_{\mathrm{op}}&\le v,&
\|\Delta\mathbf c\|_2&\le c_0.
\end{aligned}
\end{equation}
Define
\begin{equation}\label{eq:node-error-budgets}
\begin{split}
e_z&=\ell_\phi(\|\mathbf x\|_2u+b_0),\\
e_h&=\ell_\xi\bigl[e_z\|\mathbf V\|_{\mathrm{op}}
+(\|\mathbf z\|_2+e_z)v+c_0\bigr].
\end{split}
\end{equation}
The feature preactivation changes by $\mathbf x\Delta\mathbf U+\Delta\mathbf b$, and the enhancement preactivation changes by
$\Delta\mathbf z\mathbf V+(\mathbf z+\Delta\mathbf z)\Delta\mathbf V+\Delta\mathbf c$.
Lipschitz continuity and the norm inequalities give
$\|\Delta\mathbf z\|_2\le e_z$ and $\|\Delta\mathbf h\|_2\le e_h$.
Concatenation then yields the valid node budget
\begin{equation}\label{eq:node-budget}
\varepsilon_a=\sqrt{e_z^2+e_h^2}.
\end{equation}
These budgets vanish with the parameter-error budgets. The identity activation has Lipschitz constant one. When only output weights are perturbed, $\varepsilon_a=0$.

\subsubsection{Gaussian errors in the output weights}
\label{sec:gaussian-head}
Hold $\mathbf a$ fixed and let entries of $\Delta\mathbf W$ be independent $\mathcal N(0,\sigma^2)$ variables. Each coordinate of $\mathbf a\Delta\mathbf W$ is then an independent Gaussian variable with variance $\sigma^2\|\mathbf a\|_2^2$. For $\sigma\|\mathbf a\|_2>0$, the sum of their standardized squares has a chi-square distribution~\cite{lancaster2005chi}:
\begin{equation}
\frac{\|\mathbf a\Delta\mathbf W\|_2^2}{\sigma^2\|\mathbf a\|_2^2}
\sim\chi_L^2.
\end{equation}
For $\eta\in(0,1)$, let $F^{-1}_{\chi_L^2}(1-\eta)$ be its quantile. The normalized output perturbation is at most
\begin{equation}\label{eq:gaussian-radius}
r_\eta=\sigma\|\mathbf a\|_2
\sqrt{F^{-1}_{\chi_L^2}(1-\eta)/L}
\end{equation}
with probability $1-\eta$; if $\sigma\|\mathbf a\|_2=0$, it is zero.
Since $\chi_L^2/L$ concentrates at one, $r_\eta$ approaches $\sigma\|\mathbf a\|_2$ as $L$ grows, for fixed $\mathbf a$, $\sigma$, and $\eta$. The decoding condition below compares this perturbation scale with the margin of the fitted output.

\subsection{Robust Decoding under Weight Errors}
\label{sec:decoding-proof}

Assume a fixed codebook with pairwise distinct rows and a unique clean nearest-codeword decision
$c_*=\arg\min_c\|\mathbf g-\boldsymbol\omega_c\|_2$.
Define
\begin{equation}\label{eq:exact-decoding-margin}
\rho(\mathbf g;\Omega)
=\min_{c\ne c_*}
\frac{\|\mathbf g-\boldsymbol\omega_c\|_2^2
-\|\mathbf g-\boldsymbol\omega_{c_*}\|_2^2}
{2\sqrt L\,\|\boldsymbol\omega_c-\boldsymbol\omega_{c_*}\|_2}.
\end{equation}
To identify this quantity geometrically, put
$\mathbf v_c=\boldsymbol\omega_c-\boldsymbol\omega_{c_*}$ and
$\mathbf q=\mathbf g-\boldsymbol\omega_{c_*}$.
The difference of the clean squared distances is
$\|\mathbf v_c\|_2^2-2\langle\mathbf q,\mathbf v_c\rangle>0$.
The hyperplane where those distances agree has normal vector $\mathbf v_c$. Its distance from $\mathbf g$, normalized by $\sqrt L$, is
\begin{equation}\label{eq:boundary-radius}
\rho_c=
\frac{\|\mathbf v_c\|_2^2-2\langle\mathbf q,\mathbf v_c\rangle}
{2\sqrt L\,\|\mathbf v_c\|_2}.
\end{equation}
Thus $\rho=\min_{c\ne c_*}\rho_c$. In normalized output coordinates, the open ball of radius $\rho$ lies inside the clean class's Voronoi region, while its nearest boundary contains a tie. Hence $\rho$ is the radius of the largest open ball that preserves the decision.

\begin{corollary}[Preservation of the Clean Decoded Class]\label{corollary:robust-decoding}
The budgets in Theorem~\ref{thm:bounded-output-bias} preserve $c_*$ whenever
\begin{equation}\label{eq:robustness-condition}
r(\mathbf x)<\rho(\mathbf g(\mathbf x);\Omega).
\end{equation}
\end{corollary}

\begin{proof}
Let $\mathbf e=\widetilde{\mathbf g}-\mathbf g$. For any $c\ne c_*$, the difference of the two squared codeword distances changes by $-2\langle\mathbf e,\mathbf v_c\rangle$. Using Theorem~\ref{thm:bounded-output-bias}, its perturbed value is at least
$2\sqrt L\|\mathbf v_c\|_2(\rho-r)>0$.
Thus every competing codeword remains farther away.
\end{proof}

For Gaussian output-weight errors, the same argument with $r_\eta$ from \eqref{eq:gaussian-radius} shows that $r_\eta<\rho$ preserves $c_*$ with probability at least $1-\eta$.

To relate this certificate to the usual ECOC condition, define
\begin{equation}
\begin{split}
D_*&=\min_{c\ne c_*}
\frac{\|\boldsymbol\omega_c-\boldsymbol\omega_{c_*}\|_2}{\sqrt L},\\
U&=\frac{\|\mathbf g-\boldsymbol\omega_{c_*}\|_2}{\sqrt L}.
\end{split}
\end{equation}
Cauchy--Schwarz gives
\begin{equation}
\rho_c\ge\frac{\|\mathbf v_c\|_2}{2\sqrt L}
-\frac{\|\mathbf q\|_2}{\sqrt L},
\qquad
\rho\ge D_*/2-U.
\end{equation}
Consequently, $D_*/2>U+r$ is sufficient for preserving the decision, recovering the geometric condition used by Yu et al.~\cite{yu2024error}. The exact boundary distance in \eqref{eq:exact-decoding-margin} also accommodates confidence-scaled predictions that lie far from the class codewords.

For equal-norm codewords, \eqref{eq:boundary-radius} simplifies to
\begin{equation}
\rho_c=
\frac{\langle\mathbf g,\boldsymbol\omega_{c_*}-\boldsymbol\omega_c\rangle}
{\sqrt L\,\|\mathbf v_c\|_2}.
\end{equation}
Positive rescaling of an output row scales this margin by the same factor and preserves its class. For example, consider bipolar codewords and $\mathbf g=\kappa\boldsymbol\omega_{c_*}$ for $\kappa>0$. The exact radius is $\rho=\kappa D_*/2$, whereas the distance to the codeword is $U=|\kappa-1|$. As $\kappa$ increases above one, the output moves farther from its codeword and farther from every decision boundary. This distinction makes the actual boundary distance useful for analyzing predictions learned from targets of different magnitudes, as in confidence scaling.

\subsection{Connection to HiER-BLS}
\label{sec:design_implications}
\label{sec:ensemble-pruning}

The gap $\rho-r$ relates the decoding geometry to the trained model. The representation and target scaling determine the fitted outputs and their sensitivity; the codebook determines how those outputs separate the classes. In HiER-BLS, HGEE guides the feature views and target scaling of error-correcting bases (ECBs). The analysis therefore connects this guidance to the quantities that govern encoded decisions. We next express the perturbation bounds and class scores for the fused features and retained outputs used by these bases.

\subsubsection{Fused ECB features}
For fixed nonnegative fusion coefficients, write
\begin{equation}
\mathbf z=w_{\mathrm{pri}}\mathbf z_{\mathrm{pri}}
+w_{\mathrm{aux}}\mathbf z_{\mathrm{aux}}.
\end{equation}
If $\|\Delta\mathbf z_{\mathrm{pri}}\|_2\le e_{\mathrm{pri}}$ and
$\|\Delta\mathbf z_{\mathrm{aux}}\|_2\le e_{\mathrm{aux}}$, the triangle inequality gives
\begin{equation}\label{eq:fused-feature-budget}
\begin{split}
e_z&:=w_{\mathrm{pri}}e_{\mathrm{pri}}
+w_{\mathrm{aux}}e_{\mathrm{aux}},\\
\|\Delta\mathbf z\|_2&\le e_z.
\end{split}
\end{equation}
Using this $e_z$ in the enhancement bound \eqref{eq:node-error-budgets} and the concatenation formula \eqref{eq:node-budget} propagates the two subspace budgets to the broad readout.

\subsubsection{Pruned ensembles}
Consider a fixed final ensemble of $B$ fitted bases. Let $L_b$ be the retained length of base $b$. The set $\mathcal B_+=\{b\in\{1,\ldots,B\}:L_b>0\}$ indexes the bases with retained columns. Their output rows $\mathbf g_b$ concatenate to $\mathbf g$, with total code length $L_{\mathrm{tot}}=\sum_{b\in\mathcal B_+}L_b>0$. If
$\|\Delta\mathbf g_b\|_2/\sqrt{L_b}\le r_b$ for every $b\in\mathcal B_+$, then
\begin{equation}\label{eq:ensemble-output-bound}
\begin{split}
\|\Delta\mathbf g\|_2^2
&=\sum_{b\in\mathcal B_+}\|\Delta\mathbf g_b\|_2^2
\le\sum_{b\in\mathcal B_+}L_br_b^2,\\
\frac{\|\Delta\mathbf g\|_2}{\sqrt{L_{\mathrm{tot}}}}
&\le\left(
\frac{\sum_{b\in\mathcal B_+}L_br_b^2}{L_{\mathrm{tot}}}
\right)^{1/2}.
\end{split}
\end{equation}
For probabilistic block bounds, simultaneous control follows, for example, from a union bound. Corollary~\ref{corollary:robust-decoding} then applies to the concatenated codebook under its stated assumptions.

The corresponding class scores also decompose over the bases. Let $\Omega_b$ contain the retained columns for base $b$, and let $\mathbf f_{T,b}=\mathbf a_b\mathbf K_b\mathbf T$ be its unencoded score row. Target-linearity gives $\mathbf g_b=\mathbf f_{T,b}\Omega_b$, so the ensemble score row is
\begin{equation}\label{eq:ensemble-class-scores}
\mathbf s=\sum_{b\in\mathcal B_+}\mathbf g_b\Omega_b^\top
=\sum_{b\in\mathcal B_+}\mathbf f_{T,b}(\Omega_b\Omega_b^\top).
\end{equation}
Codeword correlations transform the fitted class scores from each feature view. The ensemble then sums these transformed scores. A retained column $j$ of base $b$ contributes $g_{b,j}((\Omega_b)_{c_*j}-(\Omega_b)_{cj})$ to the score difference between classes $c_*$ and $c$. A code column's value therefore depends on its learned output. This motivates PCP's joint evaluation of fitted outputs and codebooks. Retaining matching columns preserves each factorization. The perturbation bound and decoding margin can then be evaluated in the retained output space.

\section{Methodology}
\label{sec:method}
HiER-BLS couples the choice of inputs for new feature blocks with the learning of error-correcting outputs, as shown in Fig.~\ref{fig:framework}. HGEE progressively focuses new blocks on selected inputs while keeping earlier representations available to the readout. Its selected subspace size and class confidence then guide ECEO's feature views and training targets. This lets encoded learning build on the evolving branch's assessment of feature importance and sample difficulty. ECEO selects useful learned code outputs, and BADF combines the two branches' scores for the final decision.

% Architecture asset copied byte-for-byte from submission_20260623/Manuscript.
\begin{figure*}[t]
\centering
\includegraphics[width=\linewidth]{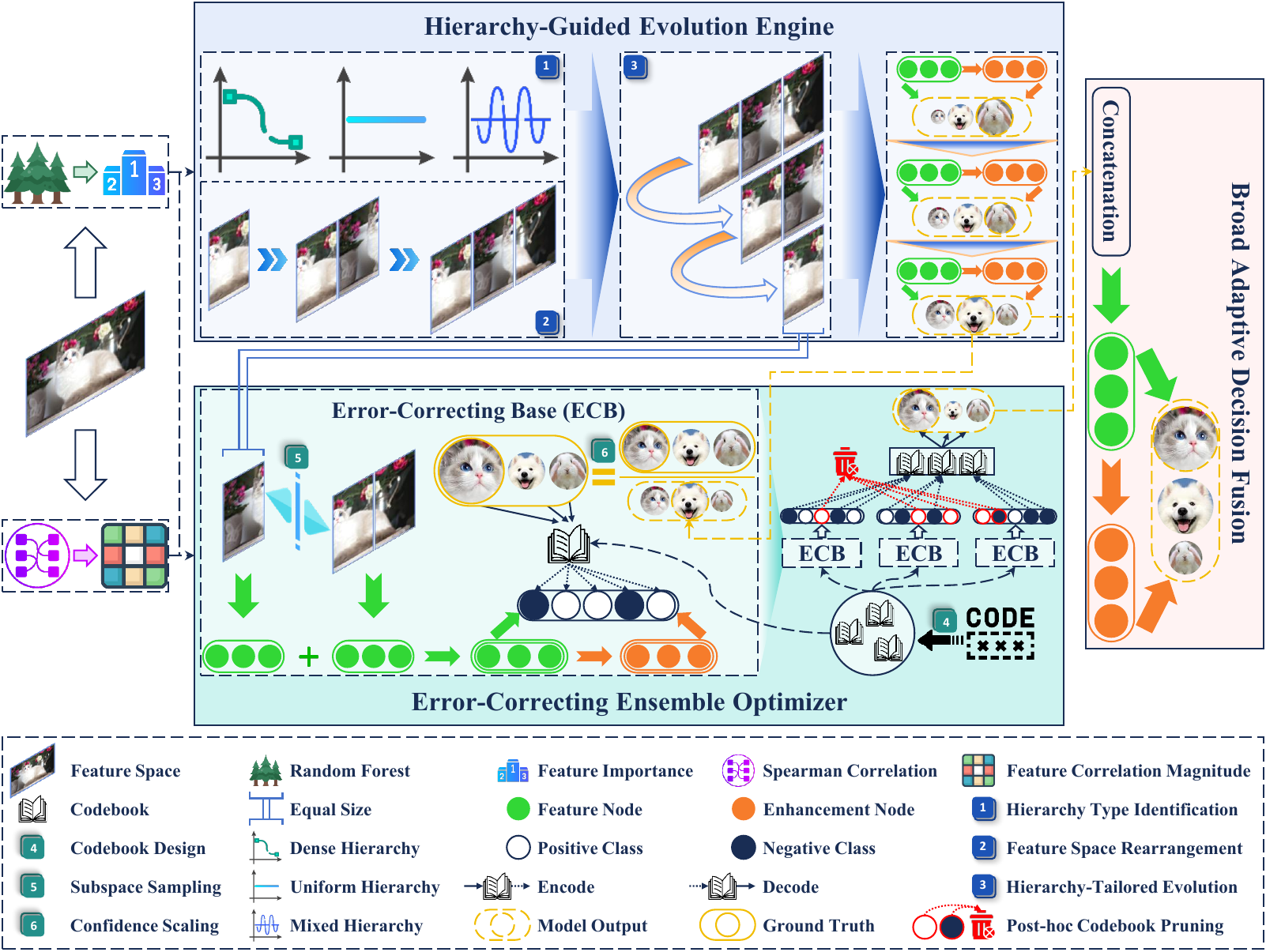}
\caption{Architecture of HiER-BLS. HGEE constructs incremental features using importance-guided subspaces. ECEO trains encoded bases and prunes output columns. BADF combines the branches' scores through an additional BLS.}
\label{fig:framework}
\end{figure*}

\subsection{Hierarchy-Guided Evolution}
\label{sec:hgee}
HGEE directs new capacity toward informative inputs with limited redundancy. Feature ordering determines which coordinates remain in later subspaces, while the importance profile guides how quickly those subspaces shrink. Earlier blocks remain in the model, so later blocks can focus on selected inputs while the readout retains access to broader representations. Let $\mathbf X\in\mathbb R^{N\times M}$ contain the training inputs, with $M>1$. Random-forest feature importance~\cite{breiman2001random} gives nonnegative values $v_i$. Define the entropy-based importance score
\begin{equation}
H=-\frac{\sum_{i=1}^M v_i\log v_i}{\log M},\qquad 0\log0:=0.
\label{eq:importance-entropy}
\end{equation}
Hierarchy Type Identification (HTI) assigns the dense type when $H\leq\tau_{\mathrm{low}}$, the uniform type when $H\geq\tau_{\mathrm{high}}$, and the mixed type otherwise, with $\tau_{\mathrm{low}}=0.1$ and $\tau_{\mathrm{high}}=0.4$. These types determine the schedule family below.

For nonconstant input columns, let $\mathbf{FCM}\in\mathbb R^{M\times M}$ contain the magnitudes of the Spearman correlations~\cite{spearman1961proof}. Feature Space Rearrangement (FSR) begins with the feature of largest importance. Given the ordered list of selected indices $\mathcal Q$ and remaining indices $\mathcal S$, define the cumulative correlation $\mathrm{LUM}_i$ and local priority $\mathrm{LP}_i$ by
\begin{align}
\mathrm{LUM}_i&=\sum_{q\in\mathcal Q}\mathbf{FCM}[i,q],\\
\mathrm{LP}_i&=\frac{\exp\!\left(v_i/(\mathrm{LUM}_i+\epsilon)\right)}
{\sum_{j\in\mathcal S}\exp\!\left(v_j/(\mathrm{LUM}_j+\epsilon)\right)},\quad i\in\mathcal S.
\end{align}
Here $\epsilon>0$ stabilizes the denominators. The largest-priority remaining feature is appended until all features are ordered. The rule greedily favors importance while penalizing correlation with already selected features. Write the reordered matrix as $\mathbf X_1=\mathbf X[:,\mathcal Q]$ and its prefix as $\mathbf X_t=\mathbf X_1[:,1:d_t]\in\mathbb R^{N\times d_t}$.

\subsubsection{Subspace schedules}
Table~\ref{tab:evolution-schedules} maps the importance profile to a subspace schedule. Concentrated profiles motivate faster truncation, whereas diffuse profiles favor retaining more coordinates in later steps.
\begin{table}[t]
\centering
\caption{Raw subspace schedules selected by the importance profile.}
\label{tab:evolution-schedules}
\setlength{\tabcolsep}{3pt}
\begin{tabular}{@{}lll@{}}
\toprule
Importance profile & Schedule & Raw size $s_t$\\
\midrule
Dense & Exponential & $M\gamma^{t-1}$\\
Uniform & Linear & $M[1-\alpha(t-1)]$\\
Mixed & Cosine & $\frac M2[1+\cos(\pi(t-1)/T)]$\\
\bottomrule
\end{tabular}
\end{table}
The parameters satisfy $0<\gamma\leq1$, $\alpha\geq0$, and $T>0$. At each admissible step,
\begin{equation}
d_t=\lfloor s_t\rfloor,\qquad 1\leq d_t\leq M,
\label{eq:schedule-boundary}
\end{equation}
and the cosine schedule uses its decreasing segment $0\leq t-1\leq T$.

\subsubsection{Incremental feature construction}
Initially, feature windows and enhancement windows are formed as
\begin{align}
\mathbf Z_{1,i}&=\phi(\mathbf X_1\mathbf W_{e_{1,i}}+\mathbf1_N\bm\beta_{e_{1,i}}),
\label{eq:feature-nodes}\\
\mathbf H_{1,j}&=\xi(\mathbf Z_1\mathbf W_{h_{1,j}}+\mathbf1_N\bm\beta_{h_{1,j}}),
\label{eq:enhancement-nodes}
\end{align}
where $\mathbf Z_1$ and $\mathbf H_1$ concatenate their respective windows. Denote their total scalar widths at step $t$ by $r_t$ and $s_t^{\mathrm h}$, and put $q_t=r_t+s_t^{\mathrm h}$.

Shrinking the input prefix changes the coordinates available to a new block. Subspace alignment uses the preceding block to construct the new mapping, carrying information from the preceding representation into this transition. At step $t$, it solves
\begin{equation}
\min_{\mathbf V_t\in\mathbb R^{r_{t-1}\times d_t}}
\|\mathbf X_t-\mathbf Z_{t-1}\mathbf V_t\|_F^2
+\lambda\|\mathbf V_t\|_1.
\label{eq:alignment}
\end{equation}
Here $\|\mathbf V_t\|_1$ is the entrywise $\ell_1$ norm, and $\lambda$ is the regularization coefficient. The new feature block is $\mathbf Z_t=\mathbf X_t\mathbf V_t^\top\in\mathbb R^{N\times r_{t-1}}$, so $r_t=r_{t-1}$. Enhancement windows are then generated from $\mathbf Z_t$ as in Eq.~\eqref{eq:enhancement-nodes}. The reconstruction step learns how the preceding features represent the retained input coordinates. The transposed coefficients then map those coordinates into a new block of the same feature width.

\subsubsection{Incremental readout learning}
Let $\mathbf F_t=[\mathbf Z_t,\mathbf H_t]\in\mathbb R^{N\times q_t}$ be the unweighted block. AdaBoost sample reweighting~\cite{freund1997decision} adjusts its training emphasis using the preceding model's class probabilities. For each sample $i$ and incorrect class $c\ne y_i$, maintain a pair weight $u_{t,ic}$, initialized as $u_{1,ic}=1/[N(C-1)]$. The sample weights and their diagonal matrix are
\begin{equation}\label{eq:sample-weight-normalization}
\begin{split}
w_{t,i}&=\frac{\sum_{c\ne y_i}u_{t,ic}}
{\sum_{k=1}^N\sum_{j\ne y_k}u_{t,kj}},\\
\mathbf D_t&=\operatorname{diag}(w_{t,1},\ldots,w_{t,N}),
\end{split}
\end{equation}
so $\mathbf D_1=\mathbf I/N$. At step $t\geq2$, set $\mathbf P^{(t-1)}=\operatorname{softmax}(\hat{\mathbf Y}_{t-1})$ rowwise. Define $\zeta_{t,ic}=u_{t-1,ic}/\sum_{j\ne y_i}u_{t-1,ij}$. The weighted error and update are
\begin{align}
e_t&=\frac12\sum_{i=1}^Nw_{t-1,i}
\left[1-P^{(t-1)}_{i,y_i}+\sum_{c\ne y_i}\zeta_{t,ic}P^{(t-1)}_{ic}\right],
\label{eq:reweight-error}\\
\beta_t&=\frac{e_t}{1-e_t},\qquad
u_{t,ic}=u_{t-1,ic}\beta_t^{[1+P^{(t-1)}_{i,y_i}-P^{(t-1)}_{ic}]/2}.
\label{eq:pair-weight-update}
\end{align}
Equation~\eqref{eq:sample-weight-normalization} then gives the new sample weights. Keeping the pair weights across steps lets the reweighting track the competing classes for each training sample.

The weighted blocks and target operators are
\begin{align}
\mathbf A_t&=\mathbf D_t^{1/2}\mathbf F_t, &
\mathbf J_t&=\mathbf F_t^\top\mathbf D_t,\\
\overline{\mathbf A}_t&=[\mathbf A_1,\ldots,\mathbf A_t], &
\overline{\mathbf J}_t&=[\mathbf J_1^\top,\ldots,\mathbf J_t^\top]^\top.
\label{eq:accumulated-blocks}
\end{align}
Here $R_t=\sum_{s=1}^tq_s$, $\overline{\mathbf A}_t\in\mathbb R^{N\times R_t}$, and $\overline{\mathbf J}_t\in\mathbb R^{R_t\times N}$. The input width $d_t$ follows the subspace schedule, while $R_t$ grows as the model accumulates blocks. Earlier blocks preserve responses to broader subspaces, and later blocks add representations of the selected prefixes. For fixed historical blocks and weights, write
\begin{align}
\mathbf G_t&=\overline{\mathbf A}_t^\top\overline{\mathbf A}_t+\lambda\mathbf I_{R_t},\\
\mathbf W_t&=\mathbf G_t^{-1}\overline{\mathbf J}_t\mathbf Y.
\label{eq:analytical-solution}
\end{align}
The unique solution of this system minimizes
\begin{equation}
\min_{\mathbf V\in\mathbb R^{R_t\times C}}
\frac12\operatorname{tr}(\mathbf V^\top\mathbf G_t\mathbf V)
-\operatorname{tr}(\mathbf V^\top\overline{\mathbf J}_t\mathbf Y).
\label{eq:block-quadratic}
\end{equation}
It is strongly convex for $\lambda>0$, because $\mathbf G_t$ is positive definite. Each block receives the supervised term $\mathbf F_s^\top\mathbf D_s\mathbf Y$ under its own sample weights. The cross terms $\mathbf F_i^\top\mathbf D_i^{1/2}\mathbf D_j^{1/2}\mathbf F_j$ couple the coefficients of different blocks. The update therefore preserves earlier sample-weight assignments while fitting the new and existing blocks jointly. When all $\mathbf D_s=\mathbf D$, Eq.~\eqref{eq:block-quadratic} reduces, up to a constant and a factor of one half, to the weighted ridge objective
$\|\mathbf D^{1/2}([\mathbf F_1,\ldots,\mathbf F_t]\mathbf V-\mathbf Y)\|_F^2+\lambda\|\mathbf V\|_F^2$.

We solve the positive-definite system using Cholesky factorization~\cite{krishnamoorthy2013matrix}. Its factor $\mathbf L_t\mathbf L_t^\top=\mathbf G_t$ can be updated as
\begin{align}
\mathbf P_t&=\mathbf A_t^\top\overline{\mathbf A}_{t-1}\mathbf L_{t-1}^{-\top},\\
\mathbf Q_t\mathbf Q_t^\top&=\mathbf A_t^\top\mathbf A_t+\lambda\mathbf I_{q_t}
-\mathbf P_t\mathbf P_t^\top,\\
\mathbf L_t&=\begin{bmatrix}\mathbf L_{t-1}&\mathbf0\\\mathbf P_t&\mathbf Q_t\end{bmatrix}.
\label{eq:cholesky-block}
\end{align}
Here $\mathbf P_t\in\mathbb R^{q_t\times R_{t-1}}$, $\mathbf Q_t\in\mathbb R^{q_t\times q_t}$, and $\mathbf L_t\in\mathbb R^{R_t\times R_t}$. The Schur complement is positive definite. Triangular solves compute $\mathbf P_t$ and solve $\mathbf L_t\mathbf U_t=\overline{\mathbf J}_t\mathbf Y$, $\mathbf L_t^\top\mathbf W_t=\mathbf U_t$.

Prediction uses the unweighted features:
\begin{equation}
\hat{\mathbf Y}_t=[\mathbf F_1,\ldots,\mathbf F_t]\mathbf W_t.
\end{equation}
For fixed historical features and sample weights, the target-to-weight operator is
\begin{equation}\label{eq:fixed-block-target-map}
\mathbf K_t=\mathbf G_t^{-1}\overline{\mathbf J}_t,
\qquad
\mathbf W_t(\mathbf T\Omega)=\mathbf W_t(\mathbf T)\Omega.
\end{equation}
Here $\mathbf W_t(\mathbf T)=\mathbf K_t\mathbf T$, so the coding identity in Section~\ref{sec:theory} applies to the accumulated representation.

For class-score rows $\mathbf U\in\mathbb R^{N_s\times C}$ with labels $\mathbf Y$, define the cross-entropy loss
\begin{equation}\label{eq:cross-entropy-loss}
\mathcal L_{\mathrm{CE}}(\mathbf U,\mathbf Y)
=-\frac1{N_s}\sum_{i=1}^{N_s}
\log\bigl[\operatorname{softmax}(\mathbf U)_{i,y_i}\bigr].
\end{equation}
HGEE monitors the training loss $\ell_t=\mathcal L_{\mathrm{CE}}(\hat{\mathbf Y}_t,\mathbf Y)$. Its relative improvement is
\begin{equation}
g_t=\frac{\ell_{t-1}-\ell_t}{\ell_{t-1}},\qquad t\geq2.
\end{equation}
Evolution stops when $g_t<\tau$. The selected size $d_{\mathrm{HG}}$ and final scores $\hat{\mathbf Y}_{\mathrm{HG}}$ are passed to ECEO.

\subsection{Error-Correcting Ensemble Optimization}
\label{sec:eceo}
HGEE explores nested prefixes of one feature ordering. ECEO uses the selected size $d_{\mathrm{HG}}$ to set the number of coordinates in each primary subspace, then samples the coordinates separately for each base. This transfers the subspace scale reached by evolution while allowing the encoded learners to draw on different feature combinations. An auxiliary map retains information from the remaining coordinates. Each error-correcting base (ECB) pairs the resulting broad feature map with a codebook sampled from a pool of $P$ random dense codes~\cite{allwein2000reducing}.

For a feature $i$, define the global correlation sum and sampling priority by
\begin{align}
\mathrm{GUM}_i&=\sum_{j\ne i}\mathbf{FCM}[i,j],\\
\mathrm{GP}_i&=\frac{\exp\!\left(v_i/(\mathrm{GUM}_i+\epsilon)\right)}
{\sum_{j=1}^M\exp\!\left(v_j/(\mathrm{GUM}_j+\epsilon)\right)}.
\end{align}
A primary subset $\mathcal S_{\mathrm{pri}}$ of size $d_{\mathrm{HG}}$ is sampled without replacement using these priorities. The remaining coordinates form the auxiliary subset. Their feature maps have equal output width and are combined as
\begin{align}
\mathbf Z_{\mathrm{fus}}&=w_{\mathrm{pri}}\mathbf Z_{\mathrm{pri}}+w_{\mathrm{aux}}\mathbf Z_{\mathrm{aux}},\\
w_{\mathrm{pri}}&=\sum_{i\in\mathcal S_{\mathrm{pri}}}\mathrm{GP}_i,\qquad w_{\mathrm{aux}}=1-w_{\mathrm{pri}}.
\end{align}
Subspace sampling and random node generation produce the different feature maps used by the bases.

\subsubsection{Confidence-scaled targets}
HGEE also identifies samples with low true-class probabilities. To direct encoded learning toward these samples, let $\mathbf P^{\mathrm{HG}}=\operatorname{softmax}(\hat{\mathbf Y}_{\mathrm{HG}})$ rowwise and scale the targets elementwise:
\begin{equation}
T_{ic}=Y_{ic}/P^{\mathrm{HG}}_{ic},\qquad \ddot{\mathbf Y}=\mathbf T\Omega.
\label{eq:confidence-scaling}
\end{equation}
For one-hot labels, this multiplies sample $i$'s true-class codeword by $1/P^{\mathrm{HG}}_{i,y_i}$. The label determines the target codeword, and HGEE's confidence determines its magnitude. Samples that HGEE finds difficult therefore have larger targets for the encoded learners. Together with subspace sampling, this gives ECEO different feature views and directs their fitting toward samples that the evolving branch handles less confidently. The scaled target is fixed when fitting the encoded readouts.

Enhancement nodes $\mathbf H_{\mathrm{fus}}$ are generated from $\mathbf Z_{\mathrm{fus}}$. For base $b$, write $\mathbf F_b=[\mathbf Z_{\mathrm{fus}},\mathbf H_{\mathrm{fus}}]\in\mathbb R^{N\times q_b}$ and let $\mathbf K_b\in\mathbb R^{q_b\times N}$ be its fixed target-to-weight operator. Then
\begin{equation}\label{eq:ecb-target-map}
\mathbf W_b=\mathbf K_b(\mathbf T\Omega_b),\qquad
\mathbf S_b=\mathbf F_b\mathbf W_b.
\end{equation}
Each base uses uniform training sample weights $\mathbf I/N$, giving
\begin{equation}\label{eq:ecb-ridge-operator}
\begin{split}
\mathbf G_b&=\mathbf F_b^\top\mathbf F_b/N+\lambda\mathbf I_{q_b},\\
\mathbf K_b&=\mathbf G_b^{-1}\mathbf F_b^\top/N.
\end{split}
\end{equation}
Cholesky factorization and triangular solves evaluate the readout. Each base is therefore a common-weight ridge fit of the form in Eq.~\eqref{eq:theory-ridge-map}.

\subsubsection{Codebook pruning}
The usefulness of a code column depends on its contribution to competing class scores through the output learned for it, as shown in Eq.~\eqref{eq:ensemble-class-scores}. PCP therefore selects columns using the validation predictions of the fitted bases. It is first applied to the initial ECB. After adding each new base, we concatenate its outputs and codebook with the retained columns. PCP then prunes the combined ensemble. At a given stage, let $\mathbf S$ contain the concatenated outputs and let $\Omega\in\{-1,1\}^{C\times L_{\mathrm{tot}}}$ be the corresponding codebook. Evaluating the fitted bases on $N_v$ validation samples gives $\mathbf S_{\mathrm{val}}\in\mathbb R^{N_v\times L_{\mathrm{tot}}}$. The gain from deleting column $j$ is
\begin{equation}
\mathrm{RG}_j=\mathrm{acc}_{\mathrm{val}}(\mathbf S_{\mathrm{val},-j},\Omega_{-j})
-\mathrm{acc}_{\mathrm{val}}(\mathbf S_{\mathrm{val}},\Omega).
\label{eq:pruning-gain}
\end{equation}
The column with the largest positive gain is removed, and the process repeats until no deletion improves validation accuracy. Both the output matrix and codebook lose the selected column. Pruning changes the accumulated class-score contributions in Eq.~\eqref{eq:ensemble-class-scores}. Column selection therefore accounts for the other bases already present.

To examine the outputs' agreement with their target codewords, we use Latent-Target Space Compatibility (LTSC). For outputs $\mathbf S$ on $N_s$ labeled samples and $L_{\mathrm{tot}}$ retained columns, it is
\begin{equation}
\mathrm{LTSC}=\frac1{N_sL_{\mathrm{tot}}}\sum_{i=1}^{N_s}\mathbf S[i]\Omega[y_i]^\top.
\label{eq:ltsc}
\end{equation}
LTSC averages true-codeword scores and therefore depends on both directional agreement and output magnitude.

After $B$ bases have been fitted and their output columns pruned, class scores are
\begin{equation}
\hat{\mathbf Y}_{\mathrm{EC}}=\mathbf S^{(B)}(\Omega^{(B)})^\top.
\end{equation}
Bipolar codewords have equal norm after common column deletion, so inner-product and nearest-Euclidean decoding agree. Starting with the second base, ECEO uses the same relative-loss stopping criterion as HGEE. It evaluates this criterion using the validation cross-entropy of the decoded class scores after pruning.

\subsection{Broad Adaptive Decision Fusion}
\label{sec:badf}
The branches produce scores for the same classes from different representations and training targets. BADF learns their combination by concatenating the scores and fitting a second-level BLS. Let $\ell_{\mathrm{HG}}$ and $\ell_{\mathrm{EC}}$ be the branches' validation cross-entropy losses. The fusion weight is $\omega=\ell_{\mathrm{EC}}/(\ell_{\mathrm{HG}}+\ell_{\mathrm{EC}})$, giving
\begin{equation}
\mathbf X_{\mathrm{BA}}=[\omega\hat{\mathbf Y}_{\mathrm{HG}},(1-\omega)\hat{\mathbf Y}_{\mathrm{EC}}]
\in\mathbb R^{N\times2C}.
\end{equation}
BADF trains an additional BLS on $(\mathbf X_{\mathrm{BA}},\mathbf Y)$ with uniform sample weights $\mathbf I/N$ to produce the final class scores.

\subsection{Computational Complexity}
\label{sec:discussion}
\label{sec:complexity}
We analyze the training, inference, and storage costs of HiER-BLS using the dense Cholesky systems in Section~\ref{sec:hgee}. Let $T_{\mathrm{HG}}$ and $B$ be the numbers of HGEE steps and fitted ECBs. The total sample count $\bar N=N+N_v$ includes training and validation samples. For a design with $q$ columns and $o$ outputs, fitting the readout and evaluating its outputs costs
\begin{equation}
\mathcal C_{\mathrm{fit}}(q,o)
=O(Nq^2+q^3+\bar Nqo+q^2o).
\end{equation}
The terms account for the training Gram matrix, factorization, target products and predictions, and triangular solves, respectively. Generating $f$ feature nodes from $d_{\mathrm{in}}$ inputs and $h$ enhancement nodes costs
\begin{equation}
\mathcal C_{\mathrm{map}}(d_{\mathrm{in}},f,h)
=O\!\left(\bar N(d_{\mathrm{in}}f+fh)\right).
\end{equation}
For subspace alignment, a direct evaluation of the data-fit gradient costs $O(Nr_{t-1}d_t)$, giving $O(I_tNr_{t-1}d_t)$ for $I_t$ first-order iterations.

Adding $q$ columns to $K$ accumulated columns updates the Cholesky factor at cost
\begin{equation}
\mathcal C_{\mathrm{up}}(K,q)
=O(NKq+Nq^2+K^2q+Kq^2+q^3).
\end{equation}
These terms correspond to the cross Gram matrix, new Gram block, triangular solve, Schur-complement product, and new factorization. The HGEE step also forms the new target product, solves for all output weights, and evaluates the accumulated model. Its total cost is therefore
\begin{equation}
\begin{split}
\mathcal C_{\mathrm{HG},t}
={}&\mathcal C_{\mathrm{up}}(R_{t-1},q_t)
+\mathcal C_{\mathrm{map}}(d_t,r_t,s_t^{\mathrm h})\\
&+O\!\left(Nq_tC+R_t^2C+\bar NR_tC
+I_tNr_{t-1}d_t\right),
\end{split}
\end{equation}
with $R_0=0$ and no alignment term at $t=1$. Sample reweighting and cross-entropy evaluation add $O(NC)$ work per step, absorbed by the displayed cost. Summing over $t=1,\ldots,T_{\mathrm{HG}}$ gives the HGEE training cost. Reusing the old factor avoids refactorizing its $R_{t-1}\times R_{t-1}$ block. Reducing $d_t$ lowers the input-mapping and alignment costs of a new block. The joint readout cost still depends on the growing accumulated width $R_t$.

\begin{table}[t]
\centering
\caption{Costs outside the HGEE updates. Here $f_b,h_b$ and $f_{\mathrm{BA}},h_{\mathrm{BA}}$ are the feature and enhancement widths of an ECB and BADF, respectively.}
\label{tab:complexity}
\small
\setlength{\tabcolsep}{3pt}
\renewcommand{\arraystretch}{1.12}
\begin{tabular}{@{}>{\raggedright\arraybackslash}p{0.31\columnwidth}>{\raggedright\arraybackslash}p{0.63\columnwidth}@{}}
\toprule
Operation & Leading cost \\
\midrule
Spearman correlations and cached FSR & $O(MN\log N+NM^2+M^2)$ \\
ECB $b$, $q_b=f_b+h_b$ & $\mathcal C_{\mathrm{map}}(M,f_b,h_b)$\newline $+\mathcal C_{\mathrm{fit}}(q_b,L_{b,0})$ \\
PCP over all added columns & $O(N_vCL_0^3)$ directly;\newline $O(N_vCL_0^2)$ with cached scores \\
BADF, $q_{\mathrm{BA}}=f_{\mathrm{BA}}+h_{\mathrm{BA}}$ & $\mathcal C_{\mathrm{map}}(2C,f_{\mathrm{BA}},h_{\mathrm{BA}})$\newline $+\mathcal C_{\mathrm{fit}}(q_{\mathrm{BA}},C)$ \\
\bottomrule
\end{tabular}
\end{table}

Table~\ref{tab:complexity} gives the remaining components. Random-forest fitting adds a cost determined by the tree count, depth, and split search. FSR uses $O(M^2)$ work when cumulative correlations are updated after each selected feature. Recomputing those sums gives $O(M^3)$. The primary and auxiliary inputs of an ECB together have $M$ coordinates, so both mappings enter its node-generation cost. The ECB costs are summed over the $B$ fitted bases. Codebook construction adds a cost determined by the random-dense code-design procedure~\cite{allwein2000reducing}.

Let $L_{b,0}$ be the number of columns introduced by base $b$, and $L_0=\sum_{b=1}^{B}L_{b,0}$. PCP operates on fixed validation outputs. With $k$ current columns, directly evaluating all $k$ deletion candidates costs $O(N_vCk^2)$. If the current class-score matrix is cached, deleting column $j$ subtracts the rank-one contribution $\mathbf S_{\mathrm{val}}[:,j]\Omega[:,j]^\top$. A candidate scan then costs $O(N_vCk)$. Each successful scan removes one column. Ensemble construction requires at most $L_0$ deletions and $B$ terminating scans. Since $B\leq L_0$, the full-process bounds are cubic and quadratic in $L_0$, respectively. Computing the ensemble's concatenated class scores after each added base also costs at most $O(\bar NCBL_0)$.

HGEE stores $O(\bar NR+R^2+RC)$ entries for its accumulated design, factor, and readout, where $R=R_{T_{\mathrm{HG}}}$. Inputs and feature correlations add $O(\bar NM+M^2)$. Fitting an ECB requires $O(\bar Nq_b+q_b^2+q_bL_{b,0})$ workspace. Retaining outputs and codebooks during pruning uses at most $O((\bar N+C)L_0)$ storage. The reweighting pair states, class scores, and BADF input require $O(\bar NC)$ entries. BADF fitting uses $O(\bar Nq_{\mathrm{BA}}+q_{\mathrm{BA}}^2+q_{\mathrm{BA}}C)$ workspace. These training requirements also include $O(PCL)$ entries for the codebook pool and the storage for the random forest.

Let $L_b$ be the final retained length of base $b$, with $L_{\mathrm{tot}}=\sum_bL_b$. The number of feature-mapping coefficients, up to bias terms, is
\begin{equation}
\begin{split}
Q_{\mathrm{map}}={}&\sum_{t=1}^{T_{\mathrm{HG}}}
(d_tr_t+r_ts_t^{\mathrm h})
+\sum_{b:L_b>0}(Mf_b+f_bh_b)\\
&+2Cf_{\mathrm{BA}}+f_{\mathrm{BA}}h_{\mathrm{BA}}.
\end{split}
\end{equation}
Evaluating these maps for one test sample and applying the trained readouts and decoder costs
\begin{equation}
O\!\left(Q_{\mathrm{map}}+RC+
\sum_bq_bL_b+CL_{\mathrm{tot}}+q_{\mathrm{BA}}C\right).
\end{equation}
The stored feature weights, output weights, and codebooks have the same order of size. Counting all fitted bases and their initial columns $L_{b,0}$ gives an upper bound on parameter storage during training. The workspaces above are additional. Pruning reduces both encoded readout evaluation and codeword scoring.

\section{Experiments}
\label{sec:experiments}
We compare HiER-BLS with existing BLS variants on five image and five tabular datasets. We then examine how guided subspace evolution affects incremental learning and how encoded readouts respond to weight errors. Ablation studies assess the two branches and their components.

\subsection{Experimental Setup}\label{sec:experimental-setup}
\subsubsection{Evaluation protocol}
The image datasets are MNIST~\cite{lecun1998gradient}, Fashion-MNIST~\cite{xiao2017fashion}, NORB~\cite{lecun2004learning}, CIFAR10, and CIFAR100~\cite{krizhevsky2009learning}. The tabular datasets are Shuttle~\cite{shuttle_148}, Letter~\cite{letter_recognition_59}, Pendigits~\cite{pen-based_recognition_of_handwritten_digits_81}, LED~\cite{led_display_domain_57}, and Waveform~\cite{waveform_database_generator}. The CIFAR experiments use features from a pretrained ResNet~\cite{he2016deep}.

We use one tenth of the training set for validation and hyperparameter selection. We average the HiER-BLS results over five runs with different random seeds. Dataset-independent settings are $\tau_{\mathrm{high}}=0.4$, $\tau_{\mathrm{low}}=0.1$, $\tau=0.01$, and $P=100$. The HiER-BLS experiments use Ubuntu 20.04, an Intel Xeon Gold 6226R CPU, and Python 3.8.10.

Table~\ref{tab:dataset_details} summarizes the dataset and model settings.\label{sec:reproducibility}

\newlength{\colsepofhypertab}
\setlength{\colsepofhypertab}{5.1pt}
\newlength{\rowsepofhypertab}
\setlength{\rowsepofhypertab}{4pt}

\begin{table*}[htbp]
\renewcommand{\arraystretch}{\arraystretchparam}
\caption{Dataset sizes and model settings. Feature and Enh. specify the feature- and enhancement-window configurations, respectively, as (number of windows, nodes per window).}
\label{tab:dataset_details}
\centering
        \begin{tblr}{
            colspec={*{14}{Q[c]}},
            colsep=\colsepofhypertab,
            rowsep=\rowsepofhypertab,
            hlines,
            vlines,
        }
        \SetCell[r=2]{c} Dataset & \SetCell[r=2]{c} Type & \SetCell[r=2]{c} Input & \SetCell[r=2]{c} Class & \SetCell[r=2]{c} Train & \SetCell[r=2]{c} Test & \SetCell[c=2]{c} HGEE & & \SetCell[c=3]{c} ECEO & & & \SetCell[c=2]{c} BADF & & \SetCell[r=2]{c} $\lambda$ \\
        & & & & & & Feature & Enh. & Feature & Enh. & $L$ & Feature & Enh. & \\
        MNIST & Image & (28,28,1) & 10 & 60000 & 10000 & (30,40) & (1,11000) & (10,10) & (1,11000) &20 & (20,10) & (1,30) & $10^{-6}$ \\
        Fashion-MNIST & Image & (28,28,1) & 10 & 60000 & 10000 & (100,10) & (1,9000) & (100,10) & (1,11000) &20 & (30,40) & (1,770) & $10^{-6}$ \\
        NORB & Image & (32,32,2) & 5 & 24300 & 24300 & (10,100) & (1,8000) & (30,50) & (1,8000) &10 & (50,35) & (1,10) & $10^{-3}$ \\
        CIFAR10 & Image & (32,32,3) & 10 & 50000 & 10000 & (128,20) & (1,5000) & (128,20) & (1,10000) & 20 & (20,25) & (1,22) & $10^{-7}$ \\
        CIFAR100 & Image & (32,32,3) & 100 & 50000 & 10000 & (128,20) & (1,5000) & (128,20) & (1,10000) & 200 & (10,10) & (1,10) & $10^{-6}$ \\
        Shuttle & Tabular & 10 & 7 & 43500 & 14500 & (10,10) & (1,1500) & (10,10) & (1,3000) & 14 & (10,10) & (1,25) & $10^{-6}$ \\
        Letter & Tabular & 16 & 26 & 16000 & 4000 & (10,10) & (1,1500) & (10,10) & (1,3000) & 52 & (10,10) & (1,19) & $10^{-6}$ \\
        Pendigits & Tabular & 16 & 10 & 8000 & 2992 & (10,10) & (1,300) & (10,10) & (1,500) & 20 & (10,10) & (1,20) & $10^{-6}$ \\
        LED & Tabular & 24 & 10 & 160000 & 40000 & (10,10) & (1,1200) & (10,10) & (1,3000) & 20 & (10,10) & (1,47) & $10^{-6}$ \\
        Waveform & Tabular & 40 & 3 & 4200 & 800 & (10,10) & (1,200) & (10,10) & (1,500) & 6 & (10,10) & (1,10) & $10^{-6}$ \\
        \end{tblr}
\end{table*}

\subsubsection{Comparative baselines}
The image comparisons include standard BLS and its cascaded variants: CFEBLS~\cite{chen2018universal}, CFBLS-pyramid, CEBLS-dropout, CFBLS-dense, and CEBLS-dense~\cite{zhang2020analysis}. We also compare Stacked BLS~\cite{liu2020stacked} and the convolutional models CNNBLS~\cite{yang2018cnn} and ConvBLS~\cite{lei2024convbls}. Ensemble baselines include Bagging-BLS, Stacking-BLS~\cite{yicheng2021broad}, AdaBoost-BLS, and AdaBoost-Incremental BLS (AdaBoost-IBLS). We also compare AdaBoost Stacking-Incremental BLS (AS-IBLS)~\cite{yun2024adaboost} with logistic-regression (LR), decision-tree (DT), or random-forest (RF) meta-classifiers. The tabular comparisons include Incremental BLS~\cite{chen2017broad}, TiBLS~\cite{fu2022task}, BLS-CIL~\cite{du2023class}, RI-BLS~\cite{zhong2024robust}, and InvF-BLS~\cite{chen2025efficient}.

\subsection{Comparison with BLS Baselines}\label{sec:comparison}
\subsubsection{Image datasets}
Table~\ref{tab:image-pattern-recognition} shows that HiER-BLS achieves the highest mean accuracy on four of five datasets. Its gains over the strongest competing methods are 0.05, 1.29, 1.47, and 1.10 percentage points on MNIST, Fashion-MNIST, CIFAR10, and CIFAR100, respectively. The improvement is small on MNIST, where several methods already exceed 99\%, and larger on Fashion-MNIST and the two CIFAR datasets. On NORB, HiER-BLS ranks second to ConvBLS, with a difference of 0.56 percentage points.

The comparisons also show that the benefit of architectural expansion depends on how it is organized. The cascaded variants give mixed results relative to BLS; all five are less accurate on Fashion-MNIST. Among ensemble methods, AdaBoost-IBLS improves on AdaBoost-BLS across all five datasets, supporting incremental updates within a boosting framework. HiER-BLS further improves the best results in this group. The evolution and component studies below examine how choosing the inputs to new blocks and guiding encoded learning contribute to its performance. HiER-BLS also has the smallest standard deviation on Fashion-MNIST and shares the smallest values on CIFAR10 and CIFAR100.

\begin{table*}[htbp]
\renewcommand{\arraystretch}{\arraystretchparam}
\centering
\caption{Image-classification accuracy (\%). Avg and Std denote the reported mean and standard deviation. Bold marks the largest mean or smallest standard deviation in each column.}
\label{tab:image-pattern-recognition}
\begin{tabular}{ @{}  *{12}{c} @{}}
\toprule
\toprule
\multirow{2}{*}{\textbf{Category}} & \multirow{2}{*}{\diagbox{\textbf{Method}}{\textbf{Dataset}}} & \multicolumn{2}{c}{\textbf{MNIST}} & \multicolumn{2}{c}{\textbf{Fashion-MNIST}} & \multicolumn{2}{c}{\textbf{NORB}} & \multicolumn{2}{c}{\textbf{CIFAR10}} & \multicolumn{2}{c}{\textbf{CIFAR100}} \\
\cmidrule(lr){3-4} \cmidrule(lr){5-6} \cmidrule(lr){7-8} \cmidrule(lr){9-10} \cmidrule(lr){11-12}
& & Avg (\%) & Std (\%) & Avg (\%) & Std (\%) & Avg (\%) & Std (\%) & Avg (\%) & Std (\%) & Avg (\%) & Std (\%) \\
\midrule
\midrule
Backbone & BLS & 98.74 & 0.07 & 91.39 & 0.09 & 89.27 & 0.11 & 93.14 & 0.08 & 75.64 & 0.14 \\
\midrule
\multirow{5}{*}{Cascaded} & CFEBLS & 98.83 & 0.06 & 87.13 & 0.18 & 90.02 & 0.12 & 93.15 & 0.09 & 75.36 & 0.16 \\
& CFBLS-pyramid & 98.65 & 0.10 & 89.88 & 0.21 & 87.55 & 0.24 & 93.10 & 0.11 & 75.74 & 0.18 \\
& CEBLS-dropout & 98.59 & 0.14 & 89.71 & 0.26 & 88.35 & 0.27 & 93.11 & 0.15 & 75.73 & 0.23 \\
& CFBLS-dense & 98.59 & 0.09 & 89.74 & 0.19 & 87.45 & 0.22 & 93.10 & 0.10 & 75.19 & 0.17 \\
& CEBLS-dense & 98.40 & 0.11 & 89.45 & 0.22 & 88.40 & 0.23 & 93.16 & 0.12 & 75.70 & 0.19 \\
\midrule
Stacked & Stacked BLS & 99.12 & 0.05 & 91.53 & 0.08 & 91.90 & 0.10 & 92.73 & 0.10 & 76.33 & \textbf{0.13} \\
\midrule
\multirow{2}{*}{Convolutional} & CNNBLS & 96.94 & 0.18 & 84.21 & 0.31 & 90.07 & 0.20 & 92.00 & 0.22 & 75.55 & 0.29 \\
& ConvBLS & 99.28 & 0.04 & 92.43 & 0.13 & \textbf{97.19} & 0.42 & 94.50 & 0.09 & 78.34 & 0.15 \\
\midrule
\multirow{7}{*}{Ensemble} & Bagging-BLS & 99.01 & \textbf{0.02} & 90.47 & 0.06 & 90.95 & 0.10 & 93.23 & \textbf{0.07} & 76.15 & 0.21 \\
& Stacking-BLS & 98.96 & 0.04 & 90.14 & 0.12 & 90.50 & \textbf{0.09} & 93.19 & 0.08 & 76.24 & 0.16 \\
& AdaBoost-BLS & 98.76 & 0.11 & 90.83 & 0.09 & 90.06 & 0.10 & 93.18 & 0.19 & 76.54 & 0.24 \\
& AdaBoost-IBLS & 98.82 & 0.08 & 91.14 & 0.06 & 90.78 & 0.11 & 93.27 & 0.16 & 76.66 & 0.19 \\
& AS-IBLS (LR) & 98.90 & 0.12 & 92.06 & 0.07 & 92.34 & 0.20 & 93.43 & 0.13 & 76.20 & 0.17 \\
& AS-IBLS (DT) & 98.81 & 0.13 & 91.76 & 0.09 & 91.57 & \textbf{0.09} & 93.22 & 0.17 & 76.83 & 0.23 \\
& AS-IBLS (RF) & 98.87 & 0.07 & 92.28 & 0.08 & 93.04 & 0.21 & 93.26 & \textbf{0.07} & 76.84 & 0.15 \\
\midrule
HiER & \textbf{HiER-BLS (ours)} & \textbf{99.33} & 0.04 & \textbf{93.72} & \textbf{0.05} & 96.63 & 0.29 & \textbf{95.97} & \textbf{0.07} & \textbf{79.44} & \textbf{0.13} \\
\bottomrule
\bottomrule
\end{tabular}
\end{table*}

\subsubsection{Tabular datasets}
Table~\ref{tab:tabular_datasets} extends the comparison to structured data. HiER-BLS achieves the highest accuracy on all five tabular datasets. The largest gains over the strongest competing results are 5.07 and 4.54 percentage points on LED and Waveform, respectively. The differences are smaller on Shuttle and Pendigits, where the leading methods already approach 99\%.

InvF-BLS has the smallest training times in the table. The baseline values in the first four rows also appear in Zhong et al.~\cite{zhong2024robust}. The InvF-BLS values appear in Chen et al.~\cite{chen2025efficient}. The two sources use different computing environments.

\begin{table*}[htbp]
\renewcommand{\arraystretch}{\arraystretchparam}
\centering
\caption{Testing accuracy and training time on tabular datasets.}
\label{tab:tabular_datasets}
\begin{tabular}{*{11}{c}}
\toprule
\multirow{2}{*}{\diagbox{\textbf{Method}}{\textbf{Dataset}}} & \multicolumn{2}{c}{\textbf{Shuttle}} & \multicolumn{2}{c}{\textbf{Letter}} & \multicolumn{2}{c}{\textbf{Pendigits}} & \multicolumn{2}{c}{\textbf{LED}} & \multicolumn{2}{c}{\textbf{Waveform}} \\
\cmidrule(lr){2-3} \cmidrule(lr){4-5} \cmidrule(lr){6-7} \cmidrule(lr){8-9} \cmidrule(lr){10-11}
& Acc (\%) & Time (s) & Acc (\%) & Time (s) & Acc (\%) & Time (s) & Acc (\%) & Time (s) & Acc (\%) & Time (s) \\
\midrule
Incremental BLS & 98.72 & 24.47 & 93.17 & 7.23 & 98.09 & 0.90 & 74.28 & 836.73 & 84.00 & 0.30 \\
TiBLS & 98.74 & 48.96 & 91.50 & 10.97 & 97.66 & 0.90 & 74.27 & 80.36 & 84.00 & 0.66 \\
BLS-CIL & 86.23 & 195.49 & 89.08 & 80.00 & 96.32 & 2.13 & 74.13 & 60.87 & 81.63 & 0.84 \\
RI-BLS & 98.78 & 5.94 & 94.00 & 3.43 & 98.26 & 0.20 & 74.28 & 19.19 & 84.38 & 0.12 \\
InvF-BLS & 99.03 & \textbf{0.64} & 95.30 & \textbf{0.26} & 98.90 & \textbf{0.02} & 74.08 & \textbf{4.76} & 82.80 & \textbf{0.01} \\
\textbf{HiER-BLS (ours)} & \textbf{99.76} & 12.36 & \textbf{98.20} & 5.21 & \textbf{99.41} & 0.56 & \textbf{79.35} & 128.47 & \textbf{88.92} & 0.19 \\
\bottomrule
\end{tabular}
\end{table*}

\newcommand{\halfnum}{0.495\linewidth}
\subsection{Subspace Evolution}
\label{sec:feature-hierarchy-tailored-evolution}
Fig.~\ref{fig:hierarchy-evolution-pairs} examines several reduction rates within each HTI-selected schedule family. It uses exponential evolution for MNIST, linear evolution for CIFAR100, and cosine evolution for Fashion-MNIST. The left panels show importance-value distributions. The right panels track accuracy as HGEE adds successive blocks.

FSR places important, less redundant coordinates earlier in the ordering. The evolution schedule controls how quickly later coordinates are excluded from new blocks. The full-space settings are $\gamma=1$, $\alpha=0$, and the $T\to\infty$ cosine limit. They keep the inputs of every block unchanged and omit subspace alignment. Their accuracy peaks early and then declines, whereas the shrinking-subspace settings retain more of the gains from subsequent blocks. Since HGEE keeps its earlier blocks, this result concerns how new capacity is directed: the readout retains broad input coverage while later blocks focus on selected coordinates.

The trajectories also show why the truncation rate matters. On MNIST and CIFAR100, the more aggressive settings are competitive in the early steps, but the gentler $\gamma=0.5$ and $\alpha=0.16$ settings give higher final accuracy. Retaining more coordinates leaves later blocks access to information that faster truncation has already removed from their inputs. The finite-$T$ cosine settings on Fashion-MNIST reach similar final accuracies, with no persistent advantage for the most aggressive setting. These observations favor balancing early concentration with the information available for later refinement. The FSR and HTI ablations in Table~\ref{tab:ablation} reinforce the two parts of this choice: which coordinates remain available and how quickly the new blocks' inputs narrow.

\begin{figure*}[htbp]
    \centering
    \begin{subfigure}[b]{\halfnum}
        \centering
        \includegraphics[width=\linewidth]{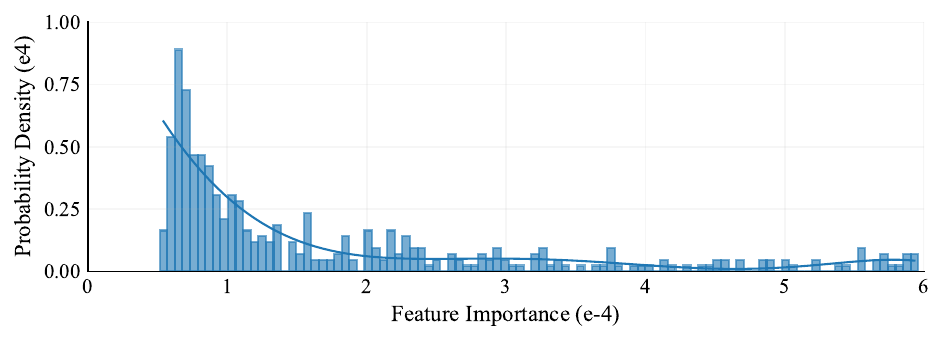}
        \caption{MNIST importance distribution (dense hierarchy)}
        \label{fig:mnist-intensive-hierarchy}
    \end{subfigure}
    \hfill
    \begin{subfigure}[b]{\halfnum}
        \centering
        \includegraphics[width=\linewidth]{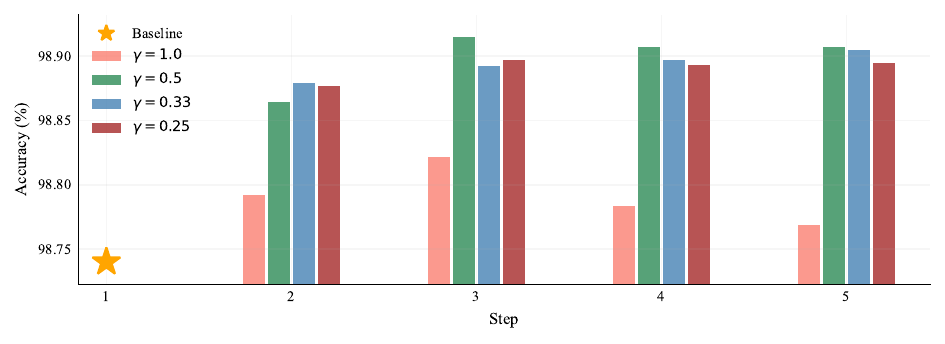}
        \caption{Exponential evolution for MNIST}
        \label{fig:mnist-exponential-evolution}
    \end{subfigure}
    \vskip\baselineskip
    \begin{subfigure}[b]{\halfnum}
        \centering
        \includegraphics[width=\linewidth]{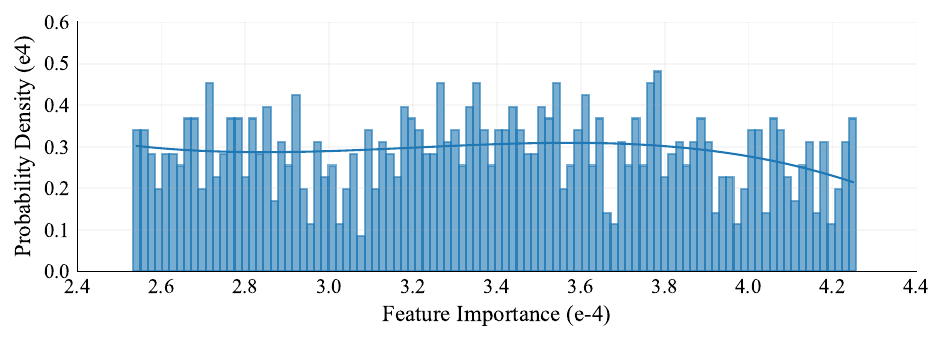}
        \caption{CIFAR100 importance distribution (uniform hierarchy)}
        \label{fig:cifar100-uniform-hierarchy}
    \end{subfigure}
    \hfill
    \begin{subfigure}[b]{\halfnum}
        \centering
        \includegraphics[width=\linewidth]{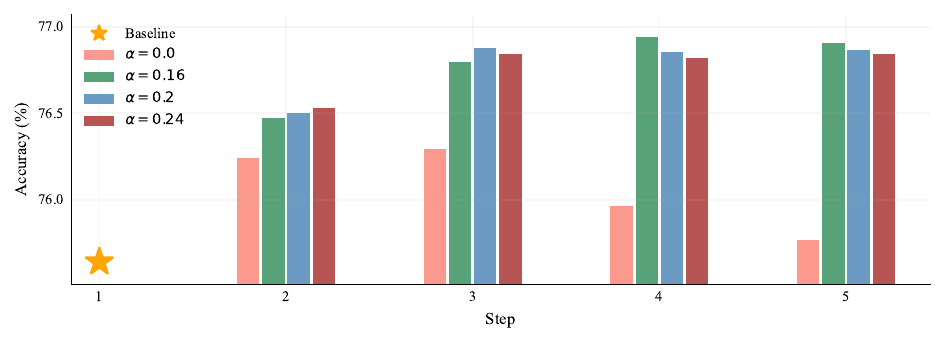}
        \caption{Linear evolution for CIFAR100}
        \label{fig:cifar100-linear-evolution}
    \end{subfigure}
    \vskip\baselineskip
    \begin{subfigure}[b]{\halfnum}
        \centering
        \includegraphics[width=\linewidth]{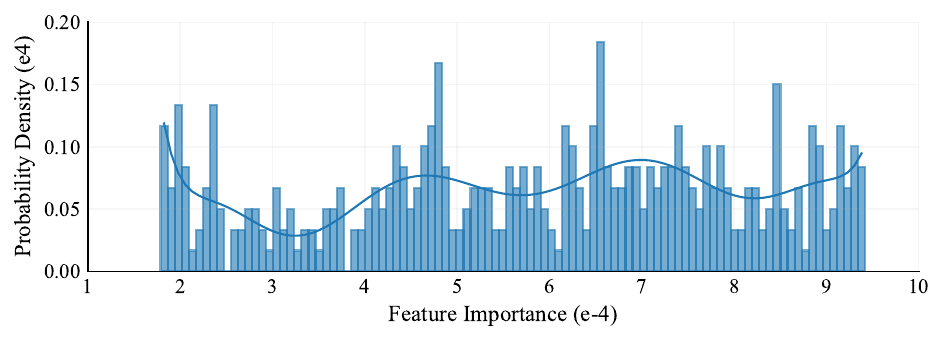}
        \caption{Fashion-MNIST importance distribution (mixed hierarchy)}
        \label{fig:fashionmnist-periodic-hierarchy}
    \end{subfigure}
    \hfill
    \begin{subfigure}[b]{\halfnum}
        \centering
        \includegraphics[width=\linewidth]{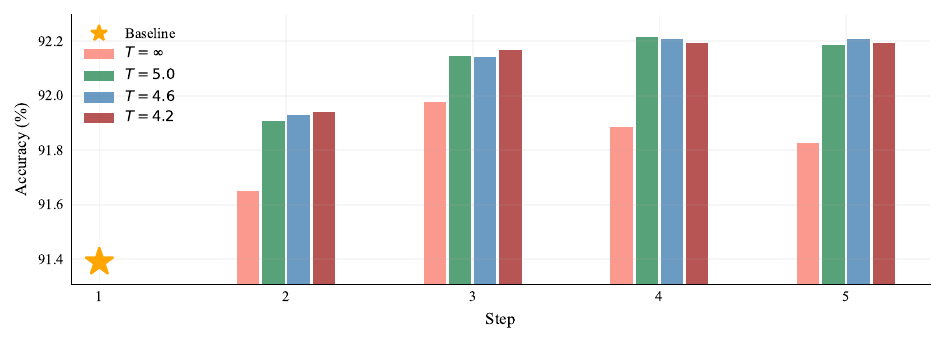}
        \caption{Cosine evolution for Fashion-MNIST}
        \label{fig:fashionmnist-cosine-evolution}
    \end{subfigure}
    \caption{Feature-importance distributions (left) and HGEE accuracy versus evolution step for selected schedule parameters (right).}
    \label{fig:hierarchy-evolution-pairs}
\end{figure*}

\subsection{Gaussian Weight-Perturbation Tests}
We perturb the trained ECEO weights with additive zero-mean Gaussian errors and vary the ECOC code length in Fig.~\ref{fig:weight-error}. This tests how the encoded branch retains classification accuracy as weight errors disturb its outputs. The curves use three fixed noise standard deviations, with horizontal one-hot baselines. The nonzero values are $\sigma=0.015,0.03$ for MNIST, $0.01,0.02$ for Fashion-MNIST and CIFAR10, and $0.001,0.002$ for CIFAR100.

Increasing code length gives the clearest gains at the stronger noise levels. Clean accuracy approaches a plateau relatively early, while the most strongly perturbed models continue to improve across a wider range of lengths. Thus, the useful range of code lengths depends on the perturbation level: redundancy can still improve performance under weight errors after its benefit to clean accuracy has become small.

At the longer tested lengths, ECEO exceeds the corresponding one-hot baseline on all four datasets. The shortest MNIST code under $\sigma=0.03$ initially falls below that baseline and overtakes it as the code grows. In the geometry of Section~\ref{sec:decoding-proof}, a prediction remains stable while its perturbed output stays within the region assigned to the clean decoded class. The observed recovery with code length is consistent with using redundant output coordinates to preserve class decisions under weight errors.

\begin{figure*}[htbp]
    \centering
    \begin{subfigure}[b]{\halfnum}
        \centering
        \includegraphics[width=\linewidth]{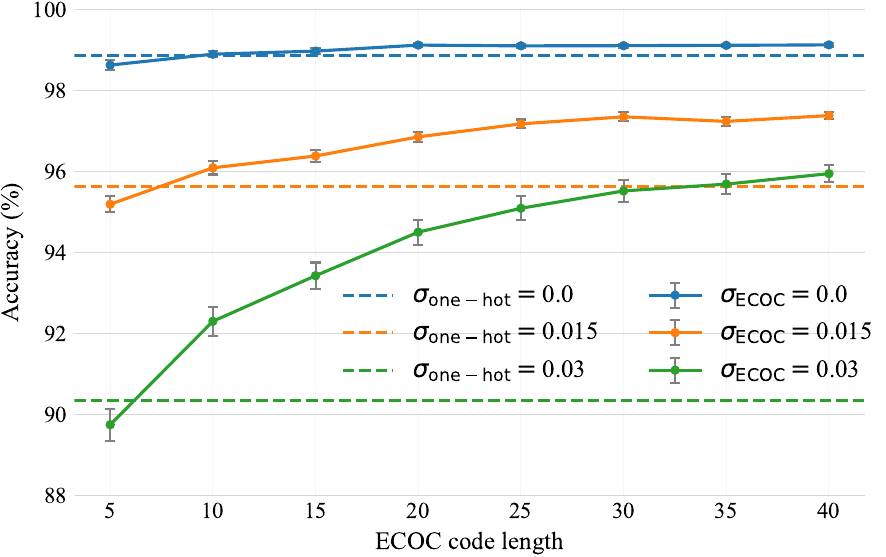}
        \caption{MNIST}
        \label{fig:mnist-error}
    \end{subfigure}
    \hfill
    \begin{subfigure}[b]{\halfnum}
        \centering
        \includegraphics[width=\linewidth]{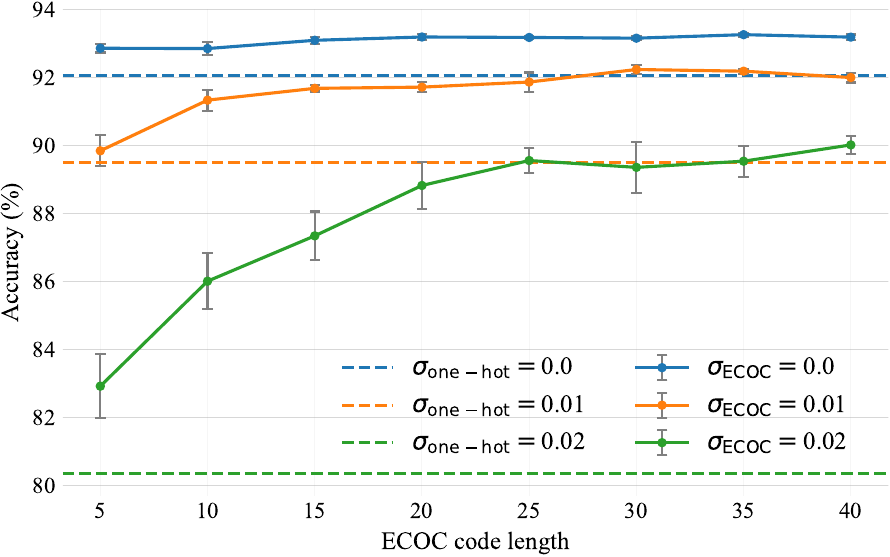}
        \caption{Fashion-MNIST}
        \label{fig:fashionmnist-error}
    \end{subfigure}
    \vskip\baselineskip
    \begin{subfigure}[b]{\halfnum}
        \centering
        \includegraphics[width=\linewidth]{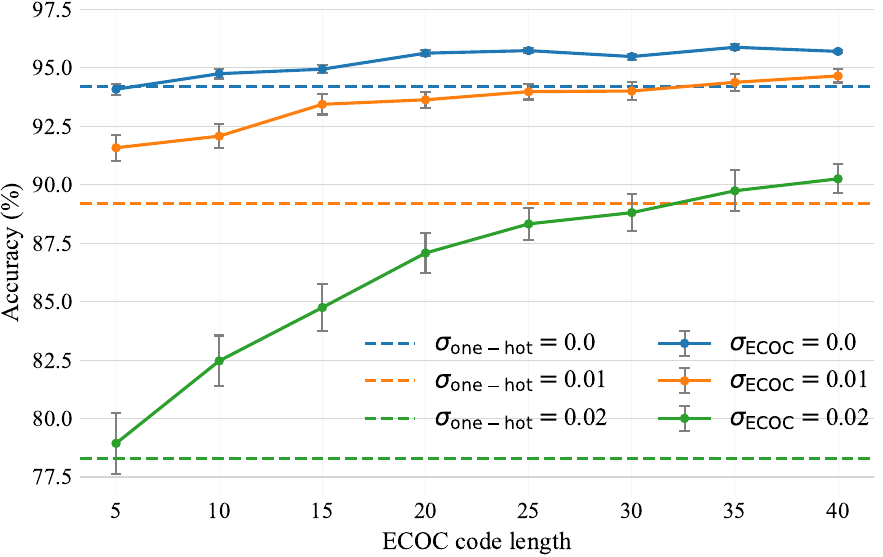}
        \caption{CIFAR10}
        \label{fig:cifar10-error}
    \end{subfigure}
    \hfill
    \begin{subfigure}[b]{\halfnum}
        \centering
        \includegraphics[width=\linewidth]{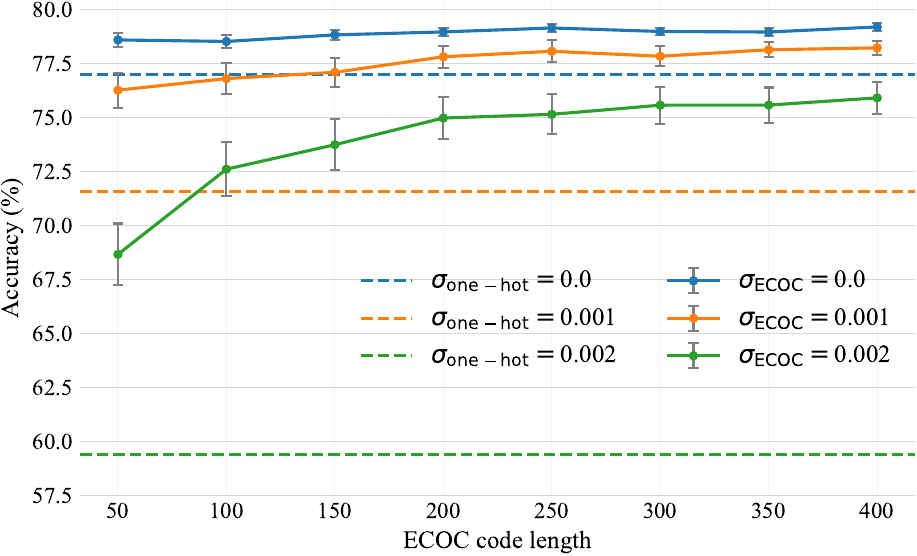}
        \caption{CIFAR100}
        \label{fig:cifar100-error}
    \end{subfigure}
    \caption{Classification accuracy versus ECOC code length under zero-mean Gaussian weight errors. Each panel includes three fixed values of $\sigma$; dashed horizontal lines denote the corresponding one-hot baselines.}
    \label{fig:weight-error}
\end{figure*}

\subsection{Ablation Study}
\subsubsection{Component comparisons}
Table~\ref{tab:ablation} examines the contributions of the branches during learning and fusion. ECEO is the stronger individual branch on all five datasets, yet BADF improves its accuracy by a further 0.18--0.53 percentage points. This gain is consistent with the two branches supplying complementary prediction information. On CIFAR100, the combined model reaches 79.44\%, exceeding HGEE and ECEO by 2.51 and 0.48 percentage points, respectively.

Within HGEE, removing FSR gives the largest accuracy decrease on every dataset. This is consistent with the prefix construction of successive subspaces: the ordering determines which coordinates remain available to later blocks. Importance and correlation guide the content of the evolving representation, while HTI selects the schedule of subspace reduction. The lower accuracies without either component complement the evolution trajectories in Fig.~\ref{fig:hierarchy-evolution-pairs}.

HGEE also improves ECEO during training. On NORB, HGEE alone reaches 91.89\%, yet its guidance raises ECEO accuracy from 94.31\% to 96.45\%. Thus, a branch with lower standalone accuracy can still improve its partner through subspace guidance and target scaling. Removing HGEE guidance lowers ECEO accuracy on all five datasets. Removing SS or CS also lowers accuracy on each dataset, supporting the use of both varied feature views and confidence-scaled supervision. Together with the further gains from BADF, these comparisons show benefits from cooperation during training and from using both branches at the decision stage. Removing PCP causes the largest ECEO decrease on four datasets; its role is examined below.

\begin{table*}[htbp]
\renewcommand{\arraystretch}{\arraystretchparam}
\centering
\caption{Testing accuracies (\%) of HiER-BLS and its ablated variants. HTI, SA, ASR, SS, and CS denote hierarchy type identification, subspace alignment, AdaBoost sample reweighting, subspace sampling, and confidence scaling. Removing HTI means randomly selecting one of the other two schedules.}
\label{tab:ablation}
\begin{tabular}{*{12}{c}  }
\toprule
\multirow{2}{*}{\textbf{Branch}} & \multirow{2}{*}{\diagbox{\textbf{Module}}{\textbf{Dataset}}} & \multicolumn{2}{c}{\textbf{MNIST}} & \multicolumn{2}{c}{\textbf{Fashion-MNIST}} & \multicolumn{2}{c}{\textbf{NORB}} & \multicolumn{2}{c}{\textbf{CIFAR10}} & \multicolumn{2}{c}{\textbf{CIFAR100}} \\
\cmidrule(lr){3-4} \cmidrule(lr){5-6} \cmidrule(lr){7-8} \cmidrule(lr){9-10} \cmidrule(lr){11-12}
& & Avg (\%) & Std (\%) & Avg (\%) & Std (\%) & Avg (\%) & Std (\%) & Avg (\%) & Std (\%) & Avg (\%) & Std (\%) \\
\midrule
\multirow{5}{*}{HGEE} & w/o HTI & 98.85 & 0.13 & 91.89 & 0.20 & 90.48 & 0.25 & 93.57 & 0.18 & 76.21 & 0.22 \\
& w/o FSR & 98.77 & 0.07 & 91.53 & \textbf{0.04} & 90.10 & 0.16 & 93.23 & 0.12 & 75.98 & 0.18 \\
& w/o SA & 98.84 & 0.11 & 92.02 & 0.17 & 90.59 & 0.22 & 93.62 & 0.19 & 76.52 & 0.24 \\
& w/o ASR & 98.87 & 0.04 & 92.00 & 0.15 & 90.72 & \textbf{0.11} & 93.66 & 0.09 & 76.64 & 0.17 \\
& Full & 98.91 & 0.05 & 92.22 & 0.09 & 91.89 & 0.15 & 93.78 & 0.13 & 76.93 & 0.14 \\
\midrule
\multirow{5}{*}{ECEO} & w/o HGEE & 98.99 & 0.06 & 92.67 & 0.15 & 94.31 & 0.35 & 95.07 & 0.10 & 78.10 & 0.17 \\
& w/o SS & 99.03 & 0.08 & 92.78 & 0.12 & 94.74 & 0.31 & 95.24 & 0.09 & 78.42 & 0.23 \\
& w/o CS & 99.08 & 0.06 & 93.06 & 0.07 & 96.05 & 0.29 & 95.45 & 0.21 & 78.74 & 0.14 \\
& w/o PCP & 99.02 & 0.11 & 92.59 & 0.16 & 94.01 & 0.40 & 94.90 & 0.13 & 77.91 & 0.19 \\
& Full & 99.12 & \textbf{0.02} & 93.19 & 0.08 & 96.45 & 0.37 & 95.63 & \textbf{0.07} & 78.96 & 0.16 \\
\midrule
HiER-BLS & Full & \textbf{99.33} & 0.04 & \textbf{93.72} & 0.05 & \textbf{96.63} & 0.29 & \textbf{95.97} & \textbf{0.07} & \textbf{79.44} & \textbf{0.13} \\
\bottomrule
\end{tabular}
\end{table*}

\subsubsection{Subspace alignment visualization}
Subspace alignment addresses how successive blocks remain related as their inputs shrink. It learns each new mapping by reconstructing the retained inputs from the preceding feature block. Fig.~\ref{fig:subspace-alignment} visualizes the resulting feature nodes on CIFAR100 with t-SNE~\cite{van2008visualizing}. Colors identify the full space and subspaces 1--4. Without alignment, the last two subspaces form groups apart from the earlier representations. With alignment, the groups lie adjacent in evolution order. This arrangement is consistent with the relationship between successive blocks encouraged by the reconstruction step. The accuracy comparisons provide complementary evidence: removing alignment lowers HGEE accuracy on every dataset, including decreases of 1.30 percentage points on NORB and 0.41 points on CIFAR100. Alongside the FSR and HTI results, this suggests that useful subspace evolution depends on both selecting later inputs and connecting their representations to earlier blocks.

\begin{figure*}[htbp]
    \centering
    \begin{subfigure}[b]{\halfnum}
        \centering
        \includegraphics[width=\linewidth]{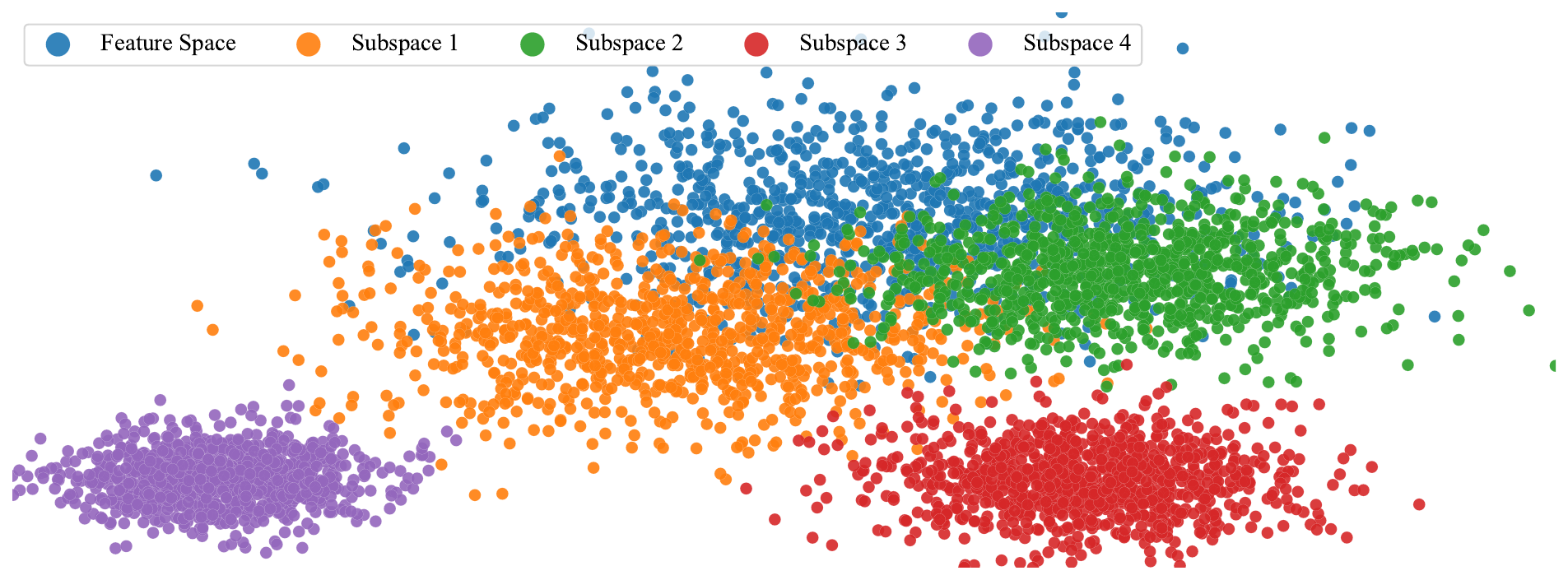}
        \caption{Visualization of feature nodes without subspace alignment.}
        \label{fig:wo-sa}
    \end{subfigure}
    \hfill
    \begin{subfigure}[b]{\halfnum}
        \centering
        \includegraphics[width=\linewidth]{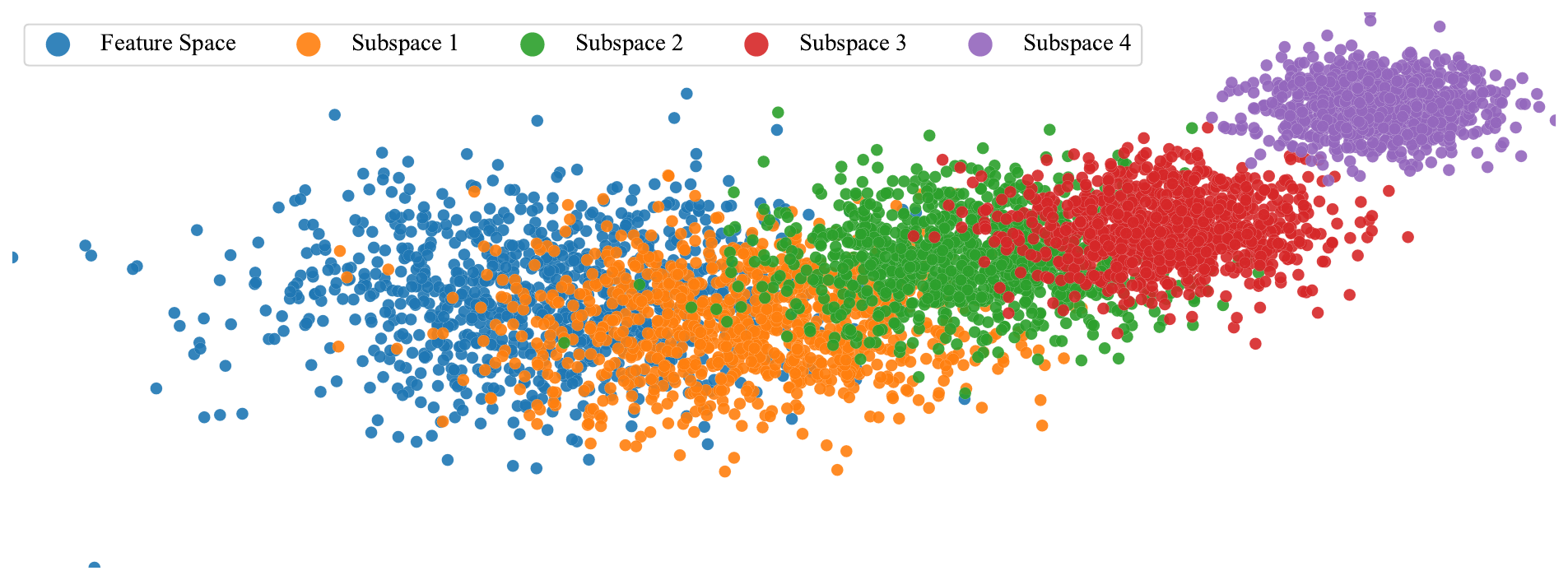}
        \caption{Visualization of feature nodes with subspace alignment.}
        \label{fig:w-sa}
    \end{subfigure}
    \caption{Two-dimensional t-SNE projections of HGEE features on CIFAR100 without and with subspace alignment. Colors identify the full feature space and subspaces 1--4.}
    \label{fig:subspace-alignment}
\end{figure*}

\subsubsection{Effect of Codebook Pruning}
\label{sec:target-space-optimization-via-pcp}
Each retained code column contributes to the decoded class scores. PCP evaluates these contributions through validation accuracy and removes columns whose deletion improves the joint prediction. In Fig.~\ref{fig:pcp-codebook-ensemble-effect}, the selected columns yield higher CIFAR10 accuracy at every ensemble stage. The advantage persists as accuracy approaches a plateau. Removing PCP lowers ECEO accuracy from 95.63\% to 94.90\% in Table~\ref{tab:ablation}.

LTSC is higher with PCP at every stage, and the gap widens as bases are added. Since LTSC averages true-class scores over samples and retained columns, these values reflect a larger mean contribution from the selected outputs. The concurrent accuracy gains show that the retained columns also work better together for classification. This agrees with Eq.~\eqref{eq:ensemble-class-scores}: each code column affects a class comparison through its learned output. The code-length and pruning studies thus identify two roles of redundancy. Additional coordinates can improve performance under weight errors, and selecting among learned coordinates can improve the ensemble's clean predictions. Effective use of ECOCs involves both providing redundancy and evaluating the contributions learned for it.

\begin{figure*}[htbp]
    \centering
    \includegraphics[width=\linewidth]{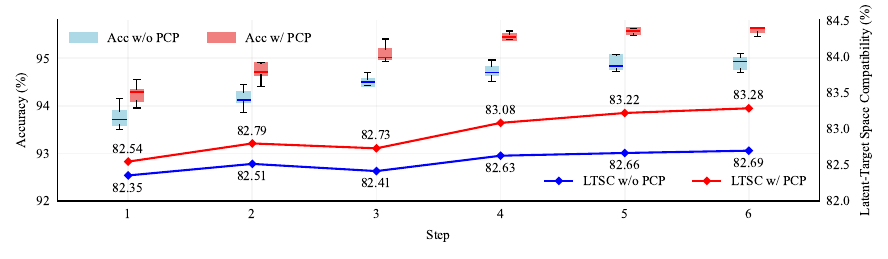}
    \caption{CIFAR10 accuracy (\%) and LTSC (\%) across ensemble stages, with and without PCP. LTSC is displayed as $100\times$ the value in Eq.~\eqref{eq:ltsc}.}
    \label{fig:pcp-codebook-ensemble-effect}
\end{figure*}

\section{Conclusion}
\label{sec:conclusion}
HiER-BLS shows the value of using incremental learning to guide both representation growth and subsequent encoded learning. The evolution trajectories and component comparisons support directing later blocks toward selected inputs while retaining the broader representations learned earlier. HGEE also improves ECEO despite having lower standalone accuracy. Its value therefore extends beyond its own predictions: the selected subspace size and sample confidence provide useful guidance to another learner. Combining the branch scores yields further gains.

The finite-readout analysis connects this learned representation to error-correcting decisions. Codeword correlations transform the fitted class scores, and prediction preservation depends on the actual output's distance to the nearest decoding boundary relative to perturbation sensitivity. This accounts for the target magnitudes used in confidence scaling. The experiments reveal two useful aspects of coding redundancy: longer codes retain value under weight errors after clean accuracy approaches a plateau, while pruning selects learned columns that improve the ensemble's predictions. Together, these findings support constructing useful representations and selecting their encoded outputs as complementary parts of broad learning. Future work will investigate budget-aware incremental computation and evaluate the framework on embedded hardware with device-specific weight errors.

\IEEEtriggercmd{\balance}
\IEEEtriggeratref{25}
\bibliographystyle{IEEEtran}
\bibliography{tcyb_reference_unique_simplified}

@article{breiman2001random,
  title={Random forests},
  author={Breiman, Leo},
  journal={Mach. Learn.},
  volume={45},
  number={1},
  pages={5--32},
  year={2001},
  publisher={Springer}
}

@misc{pen-based_recognition_of_handwritten_digits_81,
  author       = {Alpaydin, E. and Alimoglu, Fevzi.},
  title        = {{Pen-Based Recognition of Handwritten Digits}},
  year         = {1996},
  howpublished = {UCI Mach. Learn. Repository},
  note         = {{DOI}: https://doi.org/10.24432/C5MG6K}
}

@misc{waveform_database_generator,
  author       = {Breiman, L. and Stone, C.J.},
  title        = {{Waveform Database Generator (Version 1)}},
  year         = {1984},
  howpublished = {UCI Mach. Learn. Repository},
  note         = {{DOI}: https://doi.org/10.24432/C5CS3C}
}

@misc{letter_recognition_59,
  author       = {Slate, David},
  title        = {{Letter Recognition}},
  year         = {1991},
  howpublished = {UCI Mach. Learn. Repository},
  note         = {{DOI}: https://doi.org/10.24432/C5ZP40}
}

@misc{shuttle_148,
  title        = {{Statlog (Shuttle)}},
  howpublished = {UCI Mach. Learn. Repository},
  note         = {{DOI}: https://doi.org/10.24432/C5WS31}
}

@misc{led_display_domain_57,
  author       = {Breiman, L. and Friedman, J.H. and Olshen, R.A. and Stone, C.J.},
  title        = {{LED Display Domain}},
  year         = {1984},
  howpublished = {UCI Mach. Learn. Repository},
  note         = {{DOI}: https://doi.org/10.24432/C5FG61}
}

@article{yun2024adaboost,
  title={{AdaBoost}-stacking based on incremental broad learning system},
  author={Yun, Fan and Yu, Zhiwen and Yang, Kaixiang and Chen, C. L. Philip},
  journal={IEEE Trans. Knowl. Data Eng.},
  volume={36},
  number={12},
  pages={7585--7599},
  year={2024},
  publisher={IEEE}
}

@article{zhang2020analysis,
  title={Analysis and variants of broad learning system},
  author={Zhang, Liang and Li, Jiahao and Lu, Guoqing and Shen, Peiyi and Bennamoun, Mohammed and Shah, Syed Afaq Ali and Miao, Qiguang and Zhu, Guangming and Li, Ping and Lu, Xiaoyuan},
  journal={IEEE Trans. Syst., Man, Cybern., Syst.},
  volume={52},
  number={1},
  pages={334--344},
  year={2022},
  publisher={IEEE}
}

@article{pseudoinversebib, 
  title={A generalized inverse for matrices}, 
  volume={51}, 
  DOI={10.1017/S0305004100030401}, 
  number={3}, 
  journal={Math. Proc. Cambridge Philos. Soc.}, 
  author={Penrose, R.}, 
  year={1955}, 
  pages={406--413}
}

@article{chen2018universal,
  title={Universal approximation capability of broad learning system and its structural variations},
  author={Chen, C. L. Philip and Liu, Zhulin and Feng, Shuang},
  journal={IEEE Trans. Neural Netw. Learn. Syst.},
  volume={30},
  number={4},
  pages={1191--1204},
  year={2019},
  publisher={IEEE}
}

@article{liu2020stacked,
  title={Stacked broad learning system: From incremental flatted structure to deep model},
  author={Liu, Zhulin and Chen, C. L. Philip and Feng, Shuang and Feng, Qiying and Zhang, Tong},
  journal={IEEE Trans. Syst., Man, Cybern., Syst.},
  volume={51},
  number={1},
  pages={209--222},
  year={2021},
  publisher={IEEE}
}

@article{freund1997decision,
  title={A decision-theoretic generalization of on-line learning and an application to boosting},
  author={Freund, Yoav and Schapire, Robert E},
  journal={J. Comput. Syst. Sci.},
  volume={55},
  number={1},
  pages={119--139},
  year={1997},
  publisher={Elsevier}
}

@article{dietterich1994solving,
  title={Solving multiclass learning problems via error-correcting output codes},
  author={Dietterich, Thomas G and Bakiri, Ghulum},
  journal={J. Artif. Intell. Res.},
  volume={2},
  pages={263--286},
  year={1995}
}

@article{allwein2000reducing,
  title={Reducing multiclass to binary: A unifying approach for margin classifiers},
  author={Allwein, Erin L and Schapire, Robert E and Singer, Yoram},
  journal={J. Mach. Learn. Res.},
  volume={1},
  number={Dec},
  pages={113--141},
  year={2000}
}

@article{yu2024error,
  title={Error correction output codes for robust neural networks against weight-errors: A neural tangent kernel point of view},
  author={Yu, Anlan and Jing, Shusen and Lyu, Ning and Wen, Wujie and Yan, Zhiyuan},
  journal={Adv. Neural Inf. Process. Syst.},
  volume={37},
  pages={82493--82513},
  year={2024}
}

@incollection{lancaster2005chi,
  title={Chi-Square Distribution},
  author={Lancaster, Henry Oliver and Seneta, Eugene},
  booktitle={Encyclopedia of Biostatistics},
  editor={Armitage, Peter and Colton, Theodore},
  year={2005},
  publisher={Wiley},
  doi={10.1002/0470011815.b2a15018}
}

@ARTICLE{zhang2015heuristic,
  author={Zhang, Xiao-Lei},
  journal={IEEE Trans. Cybern.}, 
  title={Heuristic Ternary Error-Correcting Output Codes Via Weight Optimization and Layered Clustering-Based Approach}, 
  year={2015},
  volume={45},
  number={2},
  pages={289-301},
  doi={10.1109/TCYB.2014.2325603}}

@article{gupta2022scalable,
  title={Scalable design of error-correcting output codes using discrete optimization with graph coloring},
  author={Gupta, Samarth and Amin, Saurabh},
  journal={Adv. Neural Inf. Process. Syst.},
  volume={35},
  pages={10094--10107},
  year={2022}
}

@article{spearman1961proof,
  title={The proof and measurement of association between two things},
  author={Spearman, Charles},
  journal={Am. J. Psychol.},
  volume={15},
  number={1},
  pages={72--101},
  year={1904},
  doi={10.2307/1412159}
}

@inproceedings{dietterich2018error,
  title={Error-correcting output codes: A general method for improving multiclass inductive learning programs},
  author={Dietterich, Thomas G and Bakiri, Ghulum},
  booktitle={Proc. Natl. Conf. Artif. Intell. (AAAI)},
  pages={572--577},
  year={1991},
  publisher={AAAI Press},
  url={https://cdn.aaai.org/AAAI/1991/AAAI91-089.pdf}
}

@ARTICLE{2020FuzzyBLS,
  author={Feng, Shuang and Chen, C. L. Philip},
  journal={IEEE Trans. Cybern.}, 
  title={Fuzzy Broad Learning System: A Novel Neuro-Fuzzy Model for Regression and Classification}, 
  year={2020},
  volume={50},
  number={2},
  pages={414-424},
  doi={10.1109/TCYB.2018.2857815}}

@article{xiao2017fashion,
  title={{Fashion-MNIST}: a novel image dataset for benchmarking machine learning algorithms},
  author={Xiao, Han and Rasul, Kashif and Vollgraf, Roland},
  journal={arXiv:1708.07747},
  year={2017}
}

@techreport{krizhevsky2009learning,
  title={Learning multiple layers of features from tiny images},
  author={Krizhevsky, Alex},
  institution={University of Toronto},
  address={Toronto, ON, Canada},
  year={2009},
  url={https://www.cs.toronto.edu/~kriz/learning-features-2009-TR.pdf}
}

@inproceedings{lecun2004learning,
  title={Learning methods for generic object recognition with invariance to pose and lighting},
  author={LeCun, Yann and Huang, Fu Jie and Bottou, Leon},
  booktitle={Proc. IEEE Comput. Soc. Conf. Comput. Vis. Pattern Recognit. (CVPR 2004)},
  volume={2},
  pages={97--104},
  year={2004},
  organization={IEEE}
}

@article{gong2021research,
  title={Research review for broad learning system: Algorithms, theory, and applications},
  author={Gong, Xinrong and Zhang, Tong and Chen, C. L. Philip and Liu, Zhulin},
  journal={IEEE Trans. Cybern.},
  volume={52},
  number={9},
  pages={8922--8950},
  year={2022},
  publisher={IEEE}
}

@inproceedings{yicheng2021broad,
  title={Broad learning system based on ensemble learning},
  author={Yan, Yicheng and Guo, Wei and Wang, Liwen},
  booktitle={Proc. Int. Conf. Artif. Intell. Big Data Algorithms (CAIBDA)},
  pages={62--67},
  year={2021},
  doi={10.1109/CAIBDA53561.2021.00021},
  organization={IEEE}
}

@article{lecun1998gradient,
  title={Gradient-based learning applied to document recognition},
  author={LeCun, Yann and Bottou, L{\'e}on and Bengio, Yoshua and Haffner, Patrick},
  journal={Proc. IEEE},
  volume={86},
  number={11},
  pages={2278--2324},
  year={1998},
  publisher={Ieee}
}

@inproceedings{he2016deep,
  title={Deep residual learning for image recognition},
  author={He, Kaiming and Zhang, Xiangyu and Ren, Shaoqing and Sun, Jian},
  booktitle={Proc. IEEE Conf. Comput. Vis. Pattern Recognit.},
  pages={770--778},
  year={2016}
}

@inproceedings{yang2018cnn,
  title={A {CNN}-based broad learning system},
  author={Yang, Fan},
  booktitle={Proc. IEEE 4th Int. Conf. Comput. Commun. (ICCC)},
  pages={2105--2109},
  year={2018},
  organization={IEEE}
}

@article{lei2024convbls,
  title={{ConvBLS}: An effective and efficient incremental convolutional broad learning system combining deep and broad representations},
  author={Lei, Chunyu and Guo, Jifeng and Chen, C. L. Philip},
  journal={IEEE Trans. Artif. Intell.},
  volume={5},
  number={10},
  pages={5075--5089},
  year={2024},
  publisher={IEEE}
}

@article{chen2017broad,
  title={Broad learning system: An effective and efficient incremental learning system without the need for deep architecture},
  author={Chen, C. L. Philip and Liu, Zhulin},
  journal={IEEE Trans. Neural Netw. Learn. Syst.},
  volume={29},
  number={1},
  pages={10--24},
  year={2018},
  publisher={IEEE}
}

@article{zhong2024robust,
  title={Robust incremental broad learning system for data streams of uncertain scale},
  author={Zhong, Linjun and Chen, C. L. Philip and Guo, Jifeng and Zhang, Tong},
  journal={IEEE Trans. Neural Netw. Learn. Syst.},
  volume={36},
  number={4},
  pages={7580--7593},
  year={2025},
  publisher={IEEE}
}

@article{chen2025efficient,
  title={Efficient incremental learning for Inverse Matrix-Free broad learning system},
  author={Chen, Guang-Ze and Lei, Chunyu and Liu, Zhulin and Chen, C. L. Philip and Sun, Hai-Wei},
  journal={Inf. Fusion},
  volume={127},
  pages={103842},
  year={2026},
  publisher={Elsevier}
}

@article{du2023class,
  title={Class-incremental learning method with fast update and high retainability based on broad learning system},
  author={Du, Jie and Liu, Peng and Vong, Chi-Man and Chen, Chuangquan and Wang, Tianfu and Chen, C. L. Philip},
  journal={IEEE Trans. Neural Netw. Learn. Syst.},
  volume={35},
  number={8},
  pages={11332--11345},
  year={2024},
  publisher={IEEE}
}

@article{fu2022task,
  title={Task-incremental broad learning system for multi-component intelligent fault diagnosis of machinery},
  author={Fu, Yang and Cao, Hongrui and Chen, Xuefeng and Ding, Jianming},
  journal={Knowl.-Based Syst.},
  volume={246},
  pages={108730},
  year={2022},
  publisher={Elsevier}
}

@article{liu2024dynamic,
  title={Dynamic graph regularized broad learning with marginal fisher representation for noisy data classification},
  author={Liu, Licheng and Chen, Junhao and Liu, Tingyun and Chen, C. L. Philip and Yang, Bin},
  journal={IEEE Trans. Cybern.},
  volume={55},
  number={1},
  pages={50--63},
  year={2025},
  doi={10.1109/TCYB.2024.3471919},
  publisher={IEEE}
}

@article{deng2025robust,
  title={Robust dual-model collaborative broad learning system for classification under label noise environments},
  author={Deng, Wu and Shen, Jiuru and Ding, Jianming and Zhao, Huimin},
  journal={IEEE Internet Things J.},
  volume={12},
  number={12},
  pages={21055--21067},
  year={2025},
  publisher={IEEE}
}

@article{qian2024partial,
  title={Partial multilabel learning using noise-tolerant broad learning system with label enhancement and dimensionality reduction},
  author={Qian, Wenbin and Tu, Yanqiang and Huang, Jintao and Shu, Wenhao and Cheung, Yiu-Ming},
  journal={IEEE Trans. Neural Netw. Learn. Syst.},
  volume={36},
  number={2},
  pages={3758--3772},
  year={2025},
  publisher={IEEE}
}

@article{liu2023broad,
  title={When broad learning system meets label noise learning: A reweighting learning framework},
  author={Liu, Licheng and Chen, Junhao and Yang, Bin and Feng, Qiying and Chen, C. L. Philip},
  journal={IEEE Trans. Neural Netw. Learn. Syst.},
  volume={35},
  number={12},
  pages={18512--18524},
  year={2024},
  publisher={IEEE}
}

@article{yang2021extracting,
  title={Extracting and composing robust features with broad learning system},
  author={Yang, Kaixiang and Liu, Yuchen and Yu, Zhiwen and Chen, C. L. Philip},
  journal={IEEE Trans. Knowl. Data Eng.},
  volume={35},
  number={4},
  pages={3885--3896},
  year={2023},
  publisher={IEEE}
}

@article{yu2021progressive,
  title={Progressive ensemble kernel-based broad learning system for noisy data classification},
  author={Yu, Zhiwen and Lan, Kankan and Liu, Zhulin and Han, Guoqiang},
  journal={IEEE Trans. Cybern.},
  volume={52},
  number={9},
  pages={9656--9669},
  year={2022},
  publisher={IEEE}
}

@article{van2008visualizing,
  title={Visualizing data using {t-SNE}},
  author={van der Maaten, Laurens and Hinton, Geoffrey},
  journal={J. Mach. Learn. Res.},
  volume={9},
  number={86},
  pages={2579--2605},
  year={2008}
}

@inproceedings{krishnamoorthy2013matrix,
  title={Matrix inversion using Cholesky decomposition},
  author={Krishnamoorthy, Aravindh and Menon, Deepak},
  booktitle={Proc. Signal Process.: Algorithms Archit. Arrangements Appl. (SPA)},
  pages={70--72},
  year={2013},
  organization={IEEE}
}

@article{greville1960some,
  title={Some applications of the pseudoinverse of a matrix},
  author={Greville, Thomas NE},
  journal={SIAM Rev.},
  volume={2},
  number={1},
  pages={15--22},
  year={1960},
  publisher={SIAM}
}

@article{zhu2019ridge,
  author = {Zhu, Hufei},
  title = {Two Ridge Solutions for the Incremental Broad Learning System on Added Nodes},
  journal = {arXiv preprint arXiv:1911.04872},
  year = {2019},
  note = {Version 5, revised June 2024},
  doi = {10.48550/arXiv.1911.04872},
  url = {https://arxiv.org/abs/1911.04872v5}
}

\end{document}